\documentclass[11pt]{article} 
\pdfoutput=1
\usepackage[margin=1in]{geometry}

\usepackage{natbib}
\usepackage[utf8]{inputenc} 
\usepackage[T1]{fontenc}    
\usepackage{hyperref}       
\usepackage{url}            
\usepackage{booktabs}       
\usepackage{amsfonts}       
\usepackage{nicefrac}       
\usepackage{microtype}      
\usepackage{xcolor}         
\usepackage{amsmath}
\usepackage{amsthm}
\usepackage{graphicx}

\newtheorem{theorem}{Theorem}[section]
\newtheorem{corollary}[theorem]{Corollary}
\newtheorem{lemma}[theorem]{Lemma}
\newtheorem{fact}[theorem]{Fact}
\newtheorem{claim}[theorem]{Claim}
\theoremstyle{definition}
\newtheorem{definition}[theorem]{Definition}
\theoremstyle{remark}

\usepackage[linesnumbered,ruled]{algorithm2e}
\SetKwInput{KwInput}{Input}
\SetKwInput{KwOutput}{Output}
\usepackage{pifont} 
\newcommand{\cmark}{\ding{51}}
\newcommand{\xmark}{\ding{55}}
\usepackage{makecell}

\newcommand\A{\mathcal{A}}
\newcommand\B{\mathcal{B}}
\newcommand\D{\mathcal{D}}
\newcommand\I{\mathbb{I}}
\newcommand\PP{\mathcal{P}}
\newcommand\R{\mathcal{R}}
\newcommand\X{\mathcal{X}}
\newcommand\Z{\mathcal{Z}}
\newcommand\HH{\mathcal{H}}
\newcommand\adv{\mathsf{Adv}}
\newcommand\alg{\mathsf{Alg}}
\newcommand\err{\mathrm{error}}
\newcommand\lap{\mathrm{Lap}}
\newcommand\ld{\mathtt{LD}}
\newcommand\point{\mathtt{POINT}}
\newcommand\repd{\mathtt{RepDim}}
\newcommand\threshold{\mathtt{Threshold}}

\title{The Limits of Differential Privacy in Online Learning}
\author{Bo Li\thanks{Department of Computer Science and Engineering, HKUST. \texttt{bli@cse.ust.hk}.}
\and 
Wei Wang\thanks{Department of Computer Science and Engineering, HKUST. \texttt{weiwa@cse.ust.hk}.}
\and
Peng Ye\thanks{Department of Computer Science and Engineering, HKUST. \texttt{pyeac@connect.ust.hk}.}}

\begin{document}

\maketitle

\begin{abstract}
  Differential privacy (DP) is a formal notion that restricts the privacy leakage of an algorithm when running on sensitive data, in which privacy-utility trade-off is one of the central problems in private data analysis. In this work, we investigate the fundamental limits of differential privacy in online learning algorithms and present evidence that separates three types of constraints: no DP, pure DP, and approximate DP. We first describe a hypothesis class that is online learnable under approximate DP but not online learnable under pure DP under the adaptive adversarial setting. This indicates that approximate DP must be adopted when dealing with adaptive adversaries. We then prove that any private online learner must make an infinite number of mistakes for almost all hypothesis classes. This essentially generalizes previous results and shows a strong separation between private and non-private settings since a finite mistake bound is always attainable (as long as the class is online learnable) when there is no privacy requirement.
\end{abstract}

\section{Introduction}

Machine learning has demonstrated extraordinary capabilities in various industries, from healthcare to finance. Yet, it could raise serious privacy concerns as it may require access to a vast amount of personal data. The data used to train machine learning models may also contain sensitive information such as medical records or financial transactions. Therefore, it is crucial to ensure that the private data is well-protected during the training process.

Differential privacy (DP)~\citep{dwork2006calibrating,dwork2006our} is a rigorous mathematical definition that quantifies the level of personal data leakage. In a nutshell, an algorithm is said to be {\em differentially private} if the change of any individual's data won't make the output significantly different. DP has become the standard notion of privacy and has been broadly employed~\citep{abadi2016deep,app2017apple,abowd2018us}.

However, privacy is not a free lunch and usually comes at a cost. Simple tasks may become much harder or even intractable when privacy constraints are imposed. It is crucial to understand the cost to pay for privacy. For probably approximately correct (PAC) learning~\citep{vapnik1971uniform,blumer1989learnability,valiant1984theory}, which is the standard theoretical model of machine learning, there have been many works investigating the cost associated with privacy, which demonstrate a huge discrepancy in terms of the cost under non-private, pure private, and approximate private constraints. 

One central requirement in PAC learning is that the data need to be i.i.d.~generated and given in advance. Such assumptions fail to capture many scenarios in practice. For example, fraud detection in financial transactions often needs to be handled in real time, which prohibits access to the entire dataset. Moreover, fraudulent patterns can change over time, and new types of fraud can emerge. In such a scenario, the data are clearly not i.i.d. and can even be adaptive to the algorithm's prior predictions, in which the online learning model should be adopted.

Compared with private PAC learning, the limits of private online learning are less understood. For approximate DP, algorithms for Littlestone classes were proposed~\citep{golowich2021littlestone}. But for pure DP, the only known result is the one for point functions given by~\citet{dmitriev2024growth}. They also suggested that one could leverage existing tools of DP continual observation~\citep{dwork2010differential,jain2023price} to design algorithms for finite hypothesis classes and asked if generic learners can be constructed for infinite hypothesis classes. 

Going beyond qualitative learnability, it is worth quantitatively exploring the number of mistakes made by an online learner. Without privacy, it is possible to achieve a mistake bound of at most the Littlestone dimension of the hypothesis class~\citep{littlestone1988learning}, which is independent of the total rounds $T$. Therefore, the number of mistakes is always bounded as $T\to\infty$. However, all existing private online learning algorithms suffer from an error count that grows at least logarithmically with $T$. It was asked by~\citet{sanyal2022open} whether such a cost is inevitable for DP online learning. In a recent work of~\citet{cohen2024lower}, they showed that any private online learning algorithm for the class of point functions over $[T]$ must incur $\Omega(\log T)$ mistakes. However, it remains open whether such cost is unavoidable for generic hypothesis classes, especially for those with a smaller cardinality.

\subsection{Main Results}

We obtain results that separate three types of constraints: no DP, pure DP, and approximate DP.

\paragraph*{Separation between pure and approximate DP.} We first perform a systematic study of online learning under pure DP. We prove that every pure privately PAC learnable class is also pure privately online learnable against oblivious adversaries, answering a question raised by~\citet{dmitriev2024growth}. For the stronger adaptive adversaries, we obtain an impossibility result that the class of point functions over $\mathbb{N}$, which can be pure privately learned in the offline model, is not online learnable under pure DP.\footnote{The impossibility result can be extended to all infinite hypothesis classes due to a conclusion in~\cite{hanneke2024star} suggesting that any infinite class has either infinite Littlestone dimension or infinite star number.} According to the result of~\citet{golowich2021littlestone}, it is online learnable under approximate DP. Thus, our conclusion reveals a strong separation between these two privacy definitions.

\paragraph*{Separation between private and non-private settings.} We next quantitatively investigate the dependence on $T$ in the mistake bound. We show that for any hypothesis class $\HH$, any private online learning algorithm must make $\Omega(\log T)$ mistakes unless $\HH$ contains only one single hypothesis or exactly two complementary hypotheses (see Section~\ref{sec:lower} for the definition). This largely generalizes previous results and indicates that such a separation indeed exists universally. We further improve the lower bound to $\Omega(\ld(\HH)\log T)$, where $\ld(\HH)$ represents the Littlestone dimension of $\HH$.

To better demonstrate our results, we consider the task of online learning point functions over $\mathbb{N}$ in the oblivious setting and summarize in Table~\ref{sample-table} the finiteness of mistakes and learnability under the three types of constraints. Note that for this hypothesis class, the impossibility of making finite mistakes in private online learning can also be derived from the result in~\citep{cohen2024lower}. However, our conclusion (Theorem~\ref{thm:logT}) is more general -- it applies to a much broader family of hypothesis classes. We choose this hypothesis class for illustration because it separates the learnability against adaptive adversaries under pure DP and approximate DP.

\begin{table}
    \caption{Separation between three types of constraints}
    \label{sample-table}
    \centering
    \resizebox{\columnwidth}{!}{%
    \begin{tabular}{cccc}
      \toprule
      & \multicolumn{2}{c}{Oblivious adversary}                                & Adaptive adversary\\
      \cmidrule(lr){2-3}
      \cmidrule(lr){4-4}
      & \multicolumn{1}{c|}{Finite mistakes?} & \multicolumn{1}{c|}{Learnable?} & Learnable?  \\
      \midrule
      No DP constraints & \makecell{\cmark \\(\citep{littlestone1988learning})} & \makecell{\cmark \\(\citep{littlestone1988learning})} & \makecell{\cmark \\(\citep{littlestone1988learning})} \\
      \midrule
      Approximate DP    & \makecell{\xmark\\(Theorem~\ref{thm:logT})}  &  \makecell{\cmark\\(\citep{golowich2021littlestone})}  & \makecell{\cmark\\(\citep{golowich2021littlestone})} \\
      \midrule
      Pure DP           & \makecell{\xmark\\(Theorem~\ref{thm:logT})} & \makecell{\cmark\\(Theorem~\ref{thm:pure_oblivious_realizable})}  & \makecell{\xmark\\(Theorem~\ref{thm:pureadaptivelower})}  \\
      \bottomrule
    \end{tabular}%
    }
  \end{table}

\subsection{Related Work}

The work of~\citet{kasiviswanathan2011can} initialized the study of PAC learning with differential privacy. A series of subsequent works then showed that privacy constraints have a distinctive impact on the learners. The most remarkable result is the equivalence between approximate private learning and (non-private) online learning~\citep{alon2019private,bun2020equivalence,alon2022private,ghazi2021sample}. For pure DP, the learnability is characterized by the so-called representation dimension~\citep{beimel2019characterizing}. Both results suggest that private learning is strictly harder than non-private learning. For some specific hypothesis class such as the one-dimensional threshold over finite domain, it was shown that learning with approximate DP enjoys a much lower sample complexity than with pure DP~\citep{beimel2014bounds,chaudhuri2011sample,feldman2014sample,bun2015differentially,bun2018composable,kaplan2020privately,cohen2023optimal}, separating these two types of privacy. Another separation result is the huge gap between properly and improperly learning point functions with pure DP~\citep{beimel2014bounds}, which does not exist in non-private and approximate private settings.

The problem of private online learning Littlestone classes was first studied by~\citet{golowich2021littlestone}. They proposed private online learning algorithms for hypothesis classes of finite Littlestone dimension in the realizable setting against oblivious and adaptive adversaries, further strengthening the connection between online learning and differential privacy. In contrast with the non-private setting where the mistake bound is always finite, their algorithms exhibit a cost of $\log T$ for data streams of length $T$. Recently, it was shown by~\citet{cohen2024lower} that this extra cost is unavoidable for point functions over $[T]$. ~\citet{dmitriev2024growth} also obtained similar results, but only for algorithms that satisfy certain properties. It was questioned by~\citet{sanyal2022open} whether an unbounded number of mistakes is necessary for $\varepsilon \approx \sqrt{T}$ (see Section~\ref{sec:dp} for the definition of $\varepsilon$).

There were also a great number of results on private parametric online learning tasks such as online predictions from experts (OPE) and online convex optimization (OCO)~\citep{jain2012differentially,smith2013nearly,jain2014near,agarwal2017price,asi2023private}. Most of them focus on the agnostic setting.~\citet{asi2023near} developed algorithms for both problems in the realizable regime with oblivious adversaries, again with a $\log T$ overhead.~\citet{asi2023private} obtained some hardness results for DP-OPE against adaptive adversaries, but they require the number of experts to be larger than $T$.

Another related field is differential privacy under continual observation (see, e.g.,~\citep{dwork2010differential,chan2011private,jain2023price}). While the techniques can be used to design online learning algorithms, it is unclear whether lower bounds for DP continual observation can be transformed to any hardness results for private online learning (see~\citep{dmitriev2024growth} for a detailed discussion).
\section{Preliminaries}
\label{sec:pre}
\paragraph*{Notation.}Throughout this paper, we use $S = \{z_1,\dots,z_T\}$ to denote a data stream of length $T$. We write $S[t]$ to denote the data point comes at time-step $t$, i.e., $z_t$. For an algorithm $\A$ that runs on $S$, we use $\A(S)_t$ to denote the output of $\A$ at time-step $t$.
\subsection{Online Learning}

We start by defining online learning as a sequential game played between a learner and an adversary. Let $\HH\subseteq \{0, 1\}^{\X}$ be a hypothesis class over domain $\X$ and $T$ be a positive integer indicating the total number of rounds. In the $t$-th round, the learner outputs a hypothesis $h_t\in\{0, 1\}^{\X}$ (not required to be in $\HH$) while the adversary presents a pair $(x_t, y_t)$. The performance of the learner is measured by the expected \emph{regret}, which is the expected number of additive mistakes made by the learner compared to the best (in hindsight) hypothesis in $\HH$:
\begin{equation*}
    \mathbb{E}\left[\sum_{t=1}^T \I[h_t(x_t)\neq y_t] - \min_{h^\star\in\HH}\sum_{t=1}^T\I[h^\star(x_t)\neq y_t]\right].
\end{equation*}

The above setting is usually referred to as the \emph{agnostic} setting, where we do not make any assumptions on the data. In the \emph{realizable} setting, it is guaranteed that there is some $h^\star\in\HH$ so that $y_t = h^\star(x_t)$ for all $t\in[T]$. In this setting, the performance is directly measured by the expected number of mistakes made by the learner, which is called the \emph{mistake bound}, defined as
\begin{equation*}
    \mathbb{E}\left[\sum_{t=1}^T \I[h_t(x_t)\neq y_t]\right].
\end{equation*}

An online learning algorithm is considered successful if it attains a sublinear regret, i.e., the regret is $o(T)$. In this paper, we mainly focus on the realizable setting. We say a hypothesis class $\HH$ is online learnable if there is an online learning algorithm for $\HH$ that makes $o(T)$ mistakes in expectation.

We consider two types of adversaries: an \emph{oblivious} adversary chooses the examples in advance (could depend on the learner's strategy, but not on its internal randomness), and $(x_t,y_t)$ is revealed to the learner in the $t$-th round. An \emph{adaptive} adversary instead, can choose $(x_t,y_t)$ based on past history, i.e., $h_1,\dots,h_{t-1}$ and $(x_1,y_1),\dots,(x_{t-1},y_{t-1})$.

Without privacy, the mistake bound is exactly characterized by the Littlestone dimension~\citep{littlestone1988learning} even with stronger adversaries that can choose $(x_t,y_t)$ after seeing $h_t$. To define the Littlestone dimension, we first introduce the notion of a shattered tree.

\begin{definition}[Shattered Tree]
    Consider a full binary tree of depth $d$ such that each node is labeled by some $x\in\X$. Every $\{y_1,\dots,y_d\}\in\{0, 1\}^d$ defines a root-to-leaf path $x_1,\dots, x_d$ obtained by starting at the root, then for each $i=2,\dots, d$ choosing $x_i$ to be the left child of $x_{i-1}$ if $y_{i-1} = 0$ and to be the right child otherwise. The tree is said to be shattered by $\HH$ if for every root-to-leaf path defined in this way, there exists $h\in\HH$ such that $y_i = h(x_i)$ for all $i\in[d]$.
\end{definition}

\begin{definition}[Littlestone Dimension]
    The Littlestone dimension of $\HH$, denoted by $\ld(\HH)$, is the maximal $d$ such that there exists a full binary tree of depth $d$ that is shattered by $\HH$.
\end{definition}

The problem of online prediction from experts can be viewed as a parametric version of online learning. Let $d$ be the total number of experts. In the $t$-th round, the algorithm chooses an expert $i_t\in[d]$ while the adversary selects a loss function $\ell_t:[d] \to [0, 1]$. Then $\ell_t$ is revealed and a cost of $\ell_t(i_t)$ is incurred. The goal is to minimize the expected regret
\begin{equation*}
    \mathbb{E}\left[\sum_{t=1}^T\ell_t(i_t) - \min_{i\in[d]}\sum_{t=1}^T\ell_t(i)\right].
\end{equation*}

Similar to online learning, an oblivious adversary chooses all $\ell_t$ in advance, while an adaptive adversary determines $\ell_t$ based on $i_1,\dots,i_{t-1}$ and $\ell_1,\dots,\ell_{t-1}$. In the realizable setting, it is guaranteed that there exists $i^\star\in[d]$ such that $\ell_t(i^\star) = 0$ for all $t\in[T]$.

\subsection{Differential Privacy}
\label{sec:dp}

We first recall the standard definition of differential privacy.

\begin{definition}[Differential Privacy]
    An algorithm $\A$ is said to be $(\varepsilon,\delta)$-differentially private if for any two sequences $S_1$ and $S_2$ that differ in only one entry and any event $O$, we have
    \begin{equation*}
        \Pr[\A(S_1)\in O] \le e^\varepsilon\Pr[\A(S_2)\in O] + \delta.
    \end{equation*}
    When $\delta = 0$, we also say $\A$ is $\varepsilon$-differentially private.
\end{definition}

Our proofs use the packing argument~\citep{hardt2010geometry,beimel2014bounds}, which heavily relies on the following property of DP.

\begin{fact}[Group Privacy]
    Let $\A$ be an $(\varepsilon, \delta)$-differentially private algorithm. Then for any two sequences $S_1$ and $S_2$ that differ in $k$ entries and any event $O$, we have
    \begin{equation*}
        \Pr[\A(S_1)\in O] \le e^{k\varepsilon}\Pr[\A(S_2)\in O] + \frac{e^{k\varepsilon} - 1}{e^\varepsilon - 1}\cdot\delta.
    \end{equation*}
\end{fact}

\paragraph*{Privacy with \emph{adaptive} adversaries.} When interacting with adaptive adversaries, the notion of differential privacy becomes a bit trickier.\footnote{It turns out that the standard differential privacy and adaptive differential privacy are equivalent when $\delta = 0$~\citep{dwork2010differential,smith2013nearly}. But we will use the adaptive version in our proof since it is easier to work with.} We adopt the definition of adpative differential privacy from~\citep{golowich2021littlestone,jain2023price,asi2023private}. Let $\A$ be an online algorithm, $\adv$ be an adversary\footnote{Note that the adversary here is different from the one in the definition of online learning.} who generates two sequences $S_1$ and $S_2$ adaptively such that $S_1$ and $S_2$ differ in only one entry, and $b\in\{1, 2\}$ be a global parameter that is unknown to $\A$ and $\adv$. The interactive process $\A\circ\adv(b)$ works as follows: in each time-step $t$, $\adv$ generates two data points $S_1[t], S_2[t]$ based on the past history and $\A$ gets $S_b[t]$. The output of $\A\circ\adv(b)$ is defined to be the view of $\adv$, which includes the entire output of $\A$ and the internal randomness of $\adv$. We say $\A$ satisfies $(\varepsilon, \delta)$-adaptive differential privacy if for any such adversary $\adv$ and any event $O$, we have
\begin{equation*}
    \Pr[\A\circ\adv(1)\in O] \le e^\varepsilon\Pr[\A\circ\adv(2)\in O] + \delta.
\end{equation*}

\paragraph*{Choosing the privacy parameters.} It is a commonly agreed principle that for the definition of differential privacy to be meaningful, the parameter $\delta$ should be much less than the inverse of the dataset size~\citep{dwork2014algorithmic}. In this paper, when we say an algorithm $\A$ is private without specifying the privacy parameters, we typically refer to the set-up that $\varepsilon$ is a small constant (say $0.01$) and $\delta=o(1/T)$.

\section{Learning with Pure Differential Privacy}
\label{sec:pure}

In this section, we study online learning under pure DP constraint. We first propose algorithms for privately offline learnable hypothesis classes against oblivious adversaries via a reduction to OPE using the tool of probabilistic representation. Then we turn to adaptive adversaries and present a hypothesis class that is privately offline learnable but not privately online learnable. Note that according to the results of~\citet{golowich2021littlestone}, this class is online learnable under approximate DP with adaptive adversaries. Hence, we manifest a strong separation between pure and approximate DP.

\subsection{Learning Against Oblivious Adversaries}
\label{sec:oblivious}

In this section, we consider an oblivious adversary. We first recall the notion of representation dimension, which was introduced by~\citet{beimel2019characterizing} to characterize pure DP offline learnability. Let $\D$ be a distribution over $\X\times\{0, 1\}$ and $h\in\{0, 1\}^\X$ be a hypothesis. The error of $h$ with respect to $\D$ is defined as $\err_{\D}(h) = \Pr_{(x, y)\sim\D}[h(x)\neq y]$.

\begin{definition}[Representation Dimension]
    A probability distribution $\PP$ over hypothesis classes is said to be an $(\alpha, \beta)$-probabilistic representation for $\HH$ if for any $h^\star\in\HH$ and any distribution $\D$ over $\X\times \{0, 1\}$ that is labeled by $h^\star$, we have
    \begin{equation*}
        \Pr_{V\sim\PP}[\exists v\in V~s.t.~\err_{\D}(v) \le \alpha]\ge 1-\beta.
    \end{equation*}
    Let $\mathrm{size}(\PP) = \max_{V\in\mathrm{supp}(\PP)}\ln \lvert V \rvert$. The representation dimension of $\HH$, denoted by $\repd(\HH)$, is defined as
    \begin{equation*}
        \repd(\HH) = \min_{\PP\text{ is a }(1/4,1/4)\text{-probabilistic representation for }\HH}\mathrm{size}(\PP).
    \end{equation*}
\end{definition}

The following lemma from~\citep{beimel2019characterizing} shows that a constant probabilistic representation can be boosted to an $(\alpha, \beta)$ one with logarithmic cost in $1/\alpha$ and $1/\beta$.

\begin{lemma}
\label{lem:boosting}
    There exists an $(\alpha, \beta)$-probabilistic representation for $\HH$ with 
    \begin{equation*}
        \mathrm{size}(\PP) = O(\log(1/\alpha)\cdot(\repd(\HH) + \log\log\log(1/\alpha) + \log\log(1/\beta))).
    \end{equation*}
\end{lemma}

We first consider the realizable setting. Let $S=\{(x_1, y_1),\dots, (x_T, y_T)\}$ be the sequence chosen by the adversary and $\D_S$ be the empirical distribution of $S$ (i.e., $\Pr_{(x, y)\sim\D_S}[(x, y) = (x_t,y_t)]=1/T$ for all $t\in[T]$). By sampling a hypothesis class $V$ from an $(\alpha,\beta)$-probabilistic representation with $\alpha < 1/T$, we know that it holds with probability at least $1-\beta$ that $\err_{\D_S}(v)\le \alpha < 1/T$ for some $v\in V$. This further implies that $v$ is consistent with all examples in $S$. By Lemma~\ref{lem:boosting}, $V$ is finite as long as $\HH$ has a finite representation dimension. Thus, it suffices to run the DP-OPE algorithm in~\citep{asi2023near} with every $v\in V$ as an expert.

\begin{theorem}
\label{thm:pure_oblivious_realizable}
    Let $\HH$ be a hypothesis class with $\repd(\HH)<\infty$. In the realizable setting, there exists an online learning algorithm that is $\varepsilon$-differentially private and has an expected mistake bound of $O\left(\frac{\log^2T(\repd(\HH) + \log\log T)^2}{\varepsilon}\right)$ with an oblivious adversary.
\end{theorem}

The above conclusion directly extends to the agnostic setting by replacing the DP-OPE algorithm with an agnostic one~\citep{asi2023private}.

\begin{theorem}
\label{thm:pure_oblivious_agnostic}
    Let $\HH$ be a hypothesis class with $\repd(\HH)<\infty$. In the agnostic setting, there exists an online learning algorithm that is $\varepsilon$-differentially private and achieves an expected regret of $O\left(\frac{\sqrt{T}\log T(\repd(\HH) + \log\log T)}{\varepsilon}\right)$ with an oblivious adversary.
\end{theorem}

Note that every online learning algorithm can be transformed to a PAC learner by the online-to-batch conversion~\citep{cesa2004generalization}. Our result reveals that pure private online learnability against oblivious adversaries is equivalent to pure private PAC learnability in both realizable and agnostic settings.

\subsection{Learning Against Adaptive Adversaries}
\label{sec:adaptive}
We now turn to adaptive adversaries. For finite hypothesis classes, it is still feasible to employ techniques from DP-OPE~\citep{agarwal2017price} or DP continual observation~\citep{dwork2010differential,chan2011private,jain2023price} to devise online learning algorithms (in Appendix~\ref{sec:algo}, we give an algorithm with a better mistake bound in the realizable setting). One may hope that this can be extended to hypothesis class with finite representation dimension, as we did in the oblivious setting. However, it turns out that our method for oblivious adversaries is not applicable here. Since the examples are not fixed in advance, we cannot guarantee that the sampled hypothesis class $V$ contains a consistent hypothesis. Moreover, the famous oblivious-to-adaptive transformation (see, e.g.,~\citep{cesa2006prediction}), which was used by~\citet{golowich2021littlestone} to construct online learners under approximate DP, also fails to give a sublinear mistake bound. This is because pure DP only has the basic composition property, which yields a mistake bound that scales linearly with $T$ (for approximate DP, this can be improved to $\sqrt{T}$ by advanced composition). Therefore, it is not clear if every offline learnable hypothesis class can also be made online learnable against adaptive adversaries under pure DP.

We will show that this is an impossible mission. Let $\point_d$ be the set of point functions over $[d]$ and $\point_{\mathbb{N}}$ be the set of point functions over $\mathbb{N}$, where a point function $f_x:\X\to\{0,1\}$ is a function that maps $x$ to $1$ and all other elements to $0$. Both $\point_d$ and $\point_{\mathbb{N}}$ have a constant representation dimension and thus are offline learnable under pure DP~\citep{beimel2019characterizing}. In the rest of this section, we will prove that for any pure DP online learning algorithm for $\point_d$, an adaptive adversary can force it to make $\Omega(\min(\log d, T))$ errors. As a direct corollary, $\point_{\mathbb{N}}$ is not pure privately online learnable against adaptive adversaries. 

We now illustrate the idea of our proof. Let us start by considering a simplified version, where the algorithm is constrained to be proper, i.e., $h_t\in\HH=\point_d$ for every $t\in[T]$. Then one can construct a series of data streams $S_i=\{(i, 1),\dots,(i, 1)\}$ for every $i\in[d]$. An accurate proper learner must output $f_i$ for most of the rounds. This allows us to use the packing argument to derive an $\Omega(\log d)$ lower bound for $T = \Theta(\log d)$.

However, the above argument does not apply to the general case where the learner may be improper since a learner can simply output an all-one function that makes $0$ errors on each $S_i$. Therefore, we have to insert to $S_i$ some examples of the form $(j, 0)$ where $j\neq i$. This prevents $h_t$ from taking $1$ on elements other than $i$. But when should we insert $(j, 0)$? And how do we determine the value of $j$?

Note that till now, we have not used the adversary's adaptivity. It is necessary to exploit this power since any oblivious construction can be solved by our algorithm in Theorem~\ref{thm:pure_oblivious_realizable}. When the adversary acts adaptively, the construction becomes a dual online learning game: in each round, the learner outputs $h_t$ as a ``data point'' and the adversary chooses $(i, 1)$ or some $(j, 0)$ as the ``hypothesis''. This inspires us to leverage tools from online learning to construct the adversary.

We now sketch our idea. In each round, we choose $(i, 1)$ as the data point with probability $1/2$, and otherwise sample a $(j, 0)$ from some probability distribution. We maintain the distribution by the multiplicative update rule, which is a widely used method in online decision making. The weight of $j$ is increased by a multiplicative factor whenever $h_t(j) = 1$, and the probability of selecting $j$ is proportional to its weight. We provide a detailed implementation in Algorithm~\ref{alg:adaptive_lower}.

Using the standard argument of multiplicative update, we can show that an accurate learner must predict $h_t(i) = 1$ for most rounds and $h_t(j) = 1$ for very few rounds. This allows us to apply the packing argument to obtain the following hardness result.

\begin{theorem}
\label{thm:pureadaptivelower}
    Let $\varepsilon \le O(1)$ and $d\ge 2$. Any $\varepsilon$-differentially private online learning algorithm for $\point_d$ must incur a mistake bound of $\Omega(\min(\log d/\varepsilon, T))$ in the adaptive adversarial setting.
\end{theorem}

Since $\point_d$ is a subset of $\point_{\mathbb{N}}$ for any $d$, the above result directly implies that $\point_{\mathbb{N}}$ is not online learnable with adaptive adversaries under pure DP. This shows a strong separation between pure DP and approximate DP.

\begin{corollary}
\label{cor:pureadaptivepointN}
    Let $\varepsilon \le O(1)$. In the adaptive adversarial setting, any $\varepsilon$-differentially private online learning algorithm for $\point_{\mathbb{N}}$ must make $\Omega(T)$ mistakes.
\end{corollary}

\begin{algorithm}[!ht]
\label{alg:adaptive_lower}
\DontPrintSemicolon
    \KwInput{the number of rounds $T$; online learning algorithm $\A$; input data stream $S$}
    \KwOutput{hypotheses $h_1,\dots,h_T$ outputted by $\A$}
    $w_0(j)\gets 1$ for all $j\in[d]$\;
    \For{$t=1, \dots, T$}{
        $(x_t, y_t)\gets S[t]$\;
        Set $p(j)\gets \frac{w_{t-1}(j)}{\sum_{k\in[d]\setminus\{x_t\}}w_{t-1}(k)}$ for $j\in[d]\setminus\{x_t\}$\;
        With probability $1/2$, sample $j\sim p$ and set $(x_t, y_t)\gets (j, 0)$\;
        Present $(x_t, y_t)$ to $\A$ and receive $h_t$ from $\A$\;
        Update $w_t(j) = w_{t-1}(j)\cdot e^{h_t(j)}$ for all $j\in[d]$\;
    }
    \Return{$h_1,\dots, h_T$}
\caption{Adaptive adversary for $\point_d$}
\end{algorithm}
\section{A General Lower Bound on the Number of Mistakes}
\label{sec:lower}

In this section, we prove an $\Omega(\ld(\mathcal{H})\log T)$ lower bound on the number of mistakes made by any private learner for every hypothesis class $\HH$ that contains a pair of non-complementary hypotheses.\footnote{Such type of classes is equivalent to the notion of non-trivial classes in learning with data poisoning. See, e.g.,~\citep{bshouty2002pac} and~\citep{hanneke2022optimal}.} This implies that as $T\to\infty$, any private algorithm will make an infinite number of mistakes. Note that without privacy, the Standard Optimal Algorithm always makes at most $\ld(\HH)$ mistakes~\citep{littlestone1988learning}. Thus, our lower bound reveals a universal separation between non-private and private models.

Our proof proceeds in two steps. We first show an $\Omega(\log T)$ lower bound in Section~\ref{sec:logT}. Then based on this result, we prove the $\Omega(\ld(\HH)\log T)$ lower bound in Section~\ref{sec:LDlogT}.

\subsection{A Lower Bound for Non-complementary Hypotheses}
\label{sec:logT}

We first define the notion of complementary hypotheses.

\begin{definition}
    We say two different hypotheses $f_1$ and $f_2$ over $\X$ are complementary if $f_1(x) = 1 - f_2(x)$ for all $x\in\X$. Otherwise we say they are non-complementary.
\end{definition}

It is worth noticing the following important fact about non-complementary hypotheses, where the first item directly comes from the above definition and the second is because $f_1$ and $f_2$ are different.

\begin{fact}
\label{fact:non-complementary}
    Let $f_1$ and $f_2$ be two different hypotheses over $\X$ that are non-complementary. Then:
    \begin{enumerate}
        \item There exists some $u_0\in \X$ such that $f_1(u_0) = f_2(u_0)$;
        \item There exists some $u_1\in \X$ such that $f_1(u_1)\neq f_2(u_1)$.
    \end{enumerate}
\end{fact}

We remark that this fact is also used by~\citet{dmitriev2024growth} to prove a lower bound (in their work, they call it a ``distinguishing tuple''). However, they make a strong assumption that when running on a data stream containing $(u_0,f_1(u_0))$ only, with high probability, the algorithm predicts $h_t(u_1) = f_1(u_1)$ simultaneously for all $t\in[T]$. This largely weakens their bound since most DP algorithms clearly do not have such property.

To see how to use Fact~\ref{fact:non-complementary}, consider a hypothesis class that contains a pair of non-complementary hypotheses. We will focus on $u_0, u_1$ and $f_1, f_2$ only and ignore all other elements and hypotheses. In our proof, we will use $(u_0,f_1(u_0)) = (u_0, f_2(u_0))$ as a dummy input that provides no information about which hypothesis is correct. Let $S_0$ be a sequence that contains the dummy input only and $\A$ be an online learning algorithm. Without loss of generality, we can assume that $\Pr[\A(S_0)_t(u_1) = f_1(u_1)]\ge 1/2$ for all $t\in[T]$ (we can make this hold for half of the rounds by swapping $f_1$ and $f_2$, and ignore the rounds that it does not hold). We will insert $(u_1,f_2(u_1))$'s to make algorithm error.

Our proof relies on the classical packing argument. For ease of presentation, we only consider pure DP here, but the proof strategy easily extends to approximate DP via group privacy under approximate DP. In the framework of packing argument, we will construct a series of input sequences $S_1, \dots, S_m$ from $S_0$ and disjoint subsets of output $O_1,\dots, O_m$ such that $S_0$ and $S_i$ differ by at most $k$ elements for every $i\in[m]$, and any algorithm will make $\Omega(k)$ mistakes on $S_i$. Then by group privacy, for any $\varepsilon$-differentially private algorithm $\alg$ we have
\begin{equation*}
    1\ge \sum_{i=1}^m \Pr[\alg(S_0) \in O_i] \ge e^{-k\varepsilon}\sum_{i=1}^m \Pr[\alg(S_i)\in O_i].
\end{equation*}
Thus, a lower bound on $\Pr[\alg(S_i)\in O_i]$ implies a lower bound on $k$ by the above inequality.

The first challenge here is the construction of $S_i$. By our assumption, we can insert a $(u_1, f_2(u_1))$ at any position of $S_0$ to cause a loss of $1/2$. However, when inserting the second one, the loss may decrease by a multiplicative factor of $e^\varepsilon$. Following this argument, no matter how many $(u_1, f_2(u_1))$'s are inserted, we can only bound the expected number of mistakes by
\begin{equation*}
    \frac{1}{2}\left(1 + e^{-\varepsilon} + e^{-2\varepsilon} + \cdots\right) = constant,
\end{equation*}
failing to give an $\Omega(k)$ bound for $k = \log T$.

We overcome this by constructing them according to the given algorithm $\A$ instead of arbitrary algorithms. We will assume $\A$ has a mistake bound of $O(\log T)$ and seek to derive a contradiction. We now depict our construction. For $S_1$, we let $S = S_0$ be the initial data stream. We then go through every $t\in[T]$ in an increasing order, insert a $(u_1, f_2(u_1))$ at time-step $t$ whenever $\Pr[\A(S)_t(u_1) = f_1(u_1)]\ge 1/3$, and let $S_1 = S$ at the end. By our assumption, the number of $(u_1, f_2(u_1))$'s should not exceed $k = 3 \cdot O(\log T)$. Hence, $S_1$ and $S_0$ differ by at most $k = O(\log T)$ points. Moreover, by our construction, for each $t\in[T]$ such that $S_1[t] = (u_0, f_1(u_0))$, we must have $\Pr[\A(S_1)_t(u_1) = f_1(u_1)] < 1/3$.

Now let us construct $S_2$. We find the earliest round $t_1$ such that $\Pr[\A(S_1)_{t_1}(u_1) = f_1(u_1)] < 1/3$. The property we mentioned above ensures the existence of such $t_1$ as long as $k < T$. We then perform a similar procedure as in the construction of $S_1$, but instead of starting from $t = 1$ and going over the entire time span $[T]$, we start from $t = t_1$. The online nature of $\A$ allows us to use $t_1$ to distinguish $S_1$ and $S_2$ (as well as $S_3, \dots, S_m$, which we will construct later) since 
\begin{equation*}
    \Pr[\A(S_1)_{t_1}(u_1) = f_1(u_1)] < 1/3 < 1/2 \le \Pr[\A(S_2)_{t_1}(u_1) = f_1(u_1)].
\end{equation*}
In other words, $\A$ is more likely to predict $h_{t_1}(u_1)=f_1(u_1)$ on $S_2$ but is less likely to do so on $S_1$.

We repeat the construction for $i = 3, \dots, m$. For each $i$, we first identify the minimal $t_{i-1}$ such that $\Pr[\A(S_j)_{t_{i-1}}] < 1/3$ for every $j < i$. Then we insert $(u_1, f_2(u_1))$'s starting from $t=t_{i-1}$. By the same argument, $t_{i-1}$ can be used to distinguish $S_1, \dots, S_{i-1}$ and $S_i,\dots, S_m$. We formally describe the construction procedure in Algorithm~\ref{alg:Sandt}.

\begin{algorithm}[!ht]
\label{alg:Sandt}
\DontPrintSemicolon
    \KwInput{the number of rounds $T$; online learning algorithm $\A$; threshold $k$; $f_1, f_2$ and $u_0,u_1$}
    \KwOutput{a single data stream $S_i$, or a collection of $m$ data streams $S_1, \dots, S_m$ along with $m - 1$ time-steps $t_1, \dots, t_{m - 1}$}
    $S_0\gets\{(u_0, f_1(u_0)), \dots, (u_0, f_1(u_0))\}$\;
    $m\gets \lceil T / k\rceil$\;
    \For{$i=1, \dots, m$}{
        $S_i\gets S_0$\;
        Find the smallest $t_{i-1}$ such that $\forall j\in[i-1],\Pr[\A(S_j)_{t_{i - 1}}(u_1) = f_1(u_1)] < 1/3$ \;
        \For{$t = t_{i-1},\dots, T$}{
            \If{$\Pr[\A(S_i)_t(u_1) = f_1(u_1)] \ge 1/3$}{
                $S_i[t] \gets (u_1, f_2(u_1))$\;
            }
        }
        \If{$\left\lvert\{t\in[T]: S_i[t] = (u_1, f_2(u_1))\}\right\rvert> k$}{
            \Return{$S_i$}\;
        }
    }
    \Return{$S_1,\dots, S_m$ and $t_1,\dots, t_{m-1}$}
\caption{Constructing $S_0,S_1,\dots,S_m$ and $t_1,\dots, t_{m-1}$}
\end{algorithm}

At the end, we will have $m$ sequences $S_1,\dots, S_m$ and $m-1$ time-steps $t_1, \dots, t_{m-1}$ such that $\Pr[\A(S_i)_{t_j}(u_1) = f_1(u_1)] < 1/3$ for any $j\ge i$ and $\Pr[\A(S_i)_{t_j}(u_1) = f_1(u_1)] \ge 1/2$ for any $j < i$. It can be proved that $m = \Omega(T / k)$, which is sufficiently large for $k = O(\log T)$. Now we run $\A$ on some $S = S_i$. Suppose we can figure out the index $i$, we can apply the packing argument to derive an $\Omega(\log m) = \Omega(\log T)$ lower bound. 

Here comes the second challenge. Though we can use the output of $\A$ to estimate $\Pr[\A(S)_{t_j}(u_1) = f_1(u_1)]$ for a given $j$, we only have a constant success probability. To make the estimate accurate for every $j\in[m-1]$, one has to achieve a success probability of $1-1/m$ for each $j$. This requires running $\A$ for $O(\log m) = O(\log T)$ times and taking the average, which is prohibited since the resulting algorithm would be $O(\varepsilon\log T)$-DP, yielding a meaningless $\Omega(1)$ lower bound.

We address this issue by using binary search. We start with $\{t_1,\dots, t_{m-1}\}$ and select the middle point $t_{mid}$ in each iteration. By averaging over multiple copies of $\A(S)$, we can figure out whether we should go left or go right. This can be done in $O(\log m) = O(\log T)$ iterations, and we only require the decision made on each middle point to be correct. Thus, the number of independent copies can be reduced to $O(\log \log T)$, which leads to a lower bound of $\Omega(\log T / \log \log T)$.

The above approach is already sufficient to show an unbounded number of mistakes, but we can further refine our method to achieve an $\Omega(\log T)$ bound. The key observation here is that we do not need the probability of outputting $i$ on $S_i$ to be a constant. In fact, a success probability of $1 / m^{1-\Omega(1)}$ is enough to obtain $ k \ge \Omega(\log (m / m^{1 - \Omega(1)})) =\Omega( \log T)$.

We thus ``smooth'' our binary search. In each iteration, instead of going left or right deterministically, we go to the side that is more likely to be correct with some probability $p > 1/2$. We show that, by choosing $p$ appropriately, this approach will output $i$ on $S_i$ with probability $1 / m^{1-\Omega(1)}$. Moreover, it only requires running the online learning algorithm $O(1)$ times, avoiding the $\log \log T$ blow-up of privacy parameters. The $\Omega(\log T)$ lower bound then follows by applying the packing argument. We illustrate this approach in Algorithm~\ref{alg:smoothedbinary}.

\begin{theorem}
\label{thm:logT}
    Let $c\in(0, 1)$ be some constant. Suppose $\varepsilon\ge \ln T/ T^{1-c}$ and $\delta \le \varepsilon / T$. If $\HH$ is a hypothesis class that contains two non-complementary hypotheses, then any $(\varepsilon, \delta)$-differentially private online learning algorithm for $\HH$ must incur a mistake bound of $\Omega(\log T/\varepsilon)$ even in the oblivious adversarial setting.
\end{theorem}

One may ask whether the existence of a non-complementary pair is a necessary condition for the number of mistakes to be unbounded. Note that there are only two cases that $\HH$ contains no non-complementary pairs: either $\lvert\HH\rvert = 1$ or $\HH = \{f_1, f_2\}$ such that $f_1 = 1-f_2$. The former is definitely online learnable with zero mistakes. For the latter one, we give an algorithm with a finite expected mistake bound in Appendix~\ref{sec:algo}, showing that the condition is indeed necessary and sufficient.

\begin{algorithm}[!ht]
\label{alg:smoothedbinary}
\DontPrintSemicolon
    \KwInput{the number of rounds $T$; online learning algorithm $\A$; time-steps $t_1,\dots, t_{m-1}$; input data stream $S\in\{S_1,\dots, S_m\}$; $f_1,f_2$ and $u_0, u_1$ used in Algorithm~\ref{alg:Sandt}}
    \KwOutput{an index $i\in[m]$}
    Run $\A$ on $S$ for $360$ times, obtain $360$ copies of output $\{h_1^{(w)},\dots, h_T^{(w)}\}$ for $w\in[360]$\;
    $l\gets 1$, $r\gets m$\;
    \While{$l < r$}{
        $mid \gets \lfloor\frac{l + r}{2}\rfloor$\;
        \uIf{$\lvert\{h_{t_{mid}}^{(w)}(u_1) = f_1(u_1):w\in[360]\}\rvert < 150$}{
            Let $r\gets mid$ with probability $3/4$, and $l\gets mid + 1$ otherwise\;
        }
        \Else{
            Let $l\gets mid + 1$ with probability $3/4$, and $r\gets mid$ otherwise\;
        }
    }
\Return{$l$}\;
\caption{Distinguishing $S_1, \dots, S_m$}
\end{algorithm}

\subsection{Incorporating the Littlestone Dimension}
\label{sec:LDlogT}

Building upon the $\Omega(\log T)$ lower bound, we are now ready to show an $\Omega(\ld(\HH)\log T)$ lower bound for general hypothesis classes. Let $\A$ be a private online learning algorithm for $\HH$. Consider a shattered tree of depth $\ld(\HH) \ge 2$. Let $u_0$ denote its root and $u_1$ be its left child. By the definition of shattered tree, there exists $f_1,f_2\in\HH$ such that $f_1(u_0)=f_2(u_0)=0$ and $0=f_1(u_1)\neq f_2(u_1)=1$. Note that $f_1, f_2$ and $u_0, u_1$ satisfy the property mentioned in Fact~\ref{fact:non-complementary}. We can thus apply Theorem~\ref{thm:logT} to find a sequence $S_1$ of length $T'$ on which $\A$ makes $\Omega(\log T')$ mistakes. 

Till now, only the true labels of $u_0$ and $u_1$ are revealed to the learner. Therefore, we can go into the corresponding subtree of $u_1$ and reiterate the above operation. After repeating it $\ld(\HH)/2$ times, we obtain a series of completely non-overlapping sequences $S_1,\dots,S_{\ld(\HH)/2}$ and on any one of them $\A$ makes $\Omega(\log T')$ mistakes. By concatenating them together and setting $T' = T / \ld(\HH)$, we arrive at the $\Omega(\ld(\HH)\log T)$ lower bound assuming $T > \ld(\HH)^{1+c}$.

\begin{theorem}
\label{thm:LDlogT}
    Let $c_1\in(0, 1)$ and $c_2>0$ be two constants. Suppose $\varepsilon \ge \ln T/ T^{(1-c_1)c_2 / (1 + c_2)}$ and $\delta \le \varepsilon / T$. If $\HH$ is a hypothesis class that contains two non-complementary hypotheses, then any $(\varepsilon, \delta)$-differentially private online learning algorithm for $\HH$ must incur a mistake bound of $\Omega(\ld(\HH)\log T/\varepsilon)$ even in the oblivious adversarial setting given that $T > \ld(\HH)^{1 + c_2}$.
\end{theorem}

Note that the class of all hypotheses over $[d]$ has a Littlestone dimension of $\lfloor\log_2 d\rfloor$. The above theorem directly implies the following lower bound for the OPE problem. This improves the lower bound in~\citep{asi2023near} by a $\log T$ factor.

\begin{corollary}
\label{cor:logdlogT}
    Let $c_1\in(0, 1)$ and $c_2>0$ be two constants. Suppose $\varepsilon\ge \ln T/ T^{(1-c_1)c_2 / (1 + c_2)}$ and $\delta \le \varepsilon / T$. In the realizable setting, any $(\varepsilon, \delta)$-differentially private algorithm for OPE has a regret of $\Omega(\log d\log T/\varepsilon)$ even against oblivious adversaries given that $T > \lfloor\log_2 d\rfloor^{1 + c_2}$.

\end{corollary}

\subsection{Comparing to the Upper Bounds}
We have shown an $\Omega_{\HH}(\log T)$ lower bound on the number of mistakes made by any private online learning algorithm. We now compare it to existing upper bounds.

For pure DP, we provide an upper bound of $O_{\HH}(\log^2 T\cdot (\log\log T)^2)$. This is larger than our lower bound by a factor of $\log T\cdot (\log\log T)^2$. In Appendix~\ref{sec:algo}, we show that $O_{\HH}(\log T)$ is achievable for some specific hypothesis classes. Whether $O_{\HH}(\log T)$ is attainable for generic hypothesis classes remains open.

For approximate DP,~\citet{golowich2021littlestone} proposed an algorithm with $O_{\HH}(\log T)$ mistakes against oblivious adversaries. Thus, our lower bound is tight assuming a constant Littlestone dimension. However, their algorithm exhibits an $O_{\HH}(\sqrt{T})$ upper bound against upper bound against adaptive adversaries. Whether this can also be improved to $O_{\HH}(\log T)$ is an interesting open question.
\section{Discussion}
\label{sec:dis}

In this work, we investigate online learning with differential privacy and provide separation results that distinguish non-private, pure private, and approximate private constraints. Below, we discuss some limitations and future work.

\paragraph*{Tighter dependence on $T$ under pure DP.} Our algorithm for pure private online learning with oblivious adversaries exhibits a $\log^2 T\cdot (\log\log T)^2$ dependence on $T$ (Theorem~\ref{thm:dpOPErealizable}). We also provide algorithms for $\point_N$ (Theorem~\ref{thm:pointN}) and $\threshold_d$ (Theorem~\ref{thm:threshold}) that have a $O_{\HH}(\log T)$ mistake bound. It is interesting to find out if a $\log T$ dependence is achievable for generic hypothesis classes.

\paragraph*{Broader range of privacy parameters.} Our $\Omega(\log T)$ lower bound requires $\delta < 1/T$ (Theorem~\ref{thm:logT}) for constant $\varepsilon$, while the result in~\citep{cohen2024lower} only needs $\delta < 1/\log T$. Although it is a commonly accepted criterion to select $\delta = o(1/T)$, we still wonder whether our bound also holds for $\delta < 1/\log T$. Moreover, our results do not cover the cases where $\varepsilon$ or $T$ are extremely small. Is it possible to cover the entire range? 

\paragraph*{Mistake bound against stochastic adversaries.} One benefit of online learning is that it does not require the data to be i.i.d. generated. But in some scenarios, we may still have i.i.d. data but have to make online predictions. Clearly, such \emph{stochastic adversaries} are weaker than oblivious ones. Our construction in Algorithm~\ref{alg:Sandt} does not apply to stochastic adversaries. Can we overcome the $\Omega(\log T)$ barrier assuming stochastic adversaries?

\paragraph*{Lower bound on learning with constant success probability.} In Section~\ref{sec:logT}, we show that the \emph{expected} number of mistakes incurred by any algorithm is $\Omega(\log T)$. It is unclear whether the $\Omega(\log T)$ cost remains inevitable or a mistake bound of $o(\log T)$ can be achieved if we only require the learner to succeed with a \emph{constant probability} (e.g., $0.99$). 

\section*{Acknowledgements}
The research was supported in part by an RGC RIF grant under the contract R6021-20, an RGC TRS grant under the contract T43-513/23N-2, RGC CRF grants under the contracts C7513, C7004-22G, C1029-22G and C6015-23G, and RGC GRF grants under the contracts 16200221, 16207922 and 16207423.
The authors would like to thank the anonymous reviewers for their feedback and suggestions.

{

\bibliographystyle{plainnat}
\bibliography{references.bib}

}

\newpage
\appendix

\section{Additional Preliminaries}

\begin{theorem}[Hoeffding's Inequality~\citep{hoeffding1963probability}]
    Let $Z_1,\dots,Z_n$ be independent bounded random variables with $Z_i\in[a, b]$. Then
    \begin{equation*}
        \Pr\left[\frac{1}{n}\sum_{i=1}^n\left(Z_i - \mathbb{E}[Z_i]\right)\ge t\right]\le \exp\left(-\frac{2nt^2}{(b-a)^2}\right)
    \end{equation*}
    for all $t \ge 0$.
\end{theorem}

The Laplace mechanism ensures privacy by adding Laplace noise.

\begin{definition}[Sensitivity]
    Let $f:\Z^n\to \mathbb{R}$ be a function. The sensitivity of $f$ is defined by
    \begin{equation*}
        \Delta_f = \max_{S_1\textup{ and }S_2\textup{ differ in one entry}}\lvert f(S_1) -  f(S_2)\rvert.
    \end{equation*}
\end{definition}

\begin{definition}[Laplace Distribution]
    The Laplace distribution with parameter $b$ and mean $0$, denoted by $\lap(b)$, is defined by the following probability density function:
    \begin{equation*}
        f(x) = \frac{1}{2b}\exp(-\lvert x \rvert/ b).
    \end{equation*}
\end{definition}

\begin{lemma}[The Laplace Mechanism~\citep{dwork2006calibrating}]
    Let $r\sim\lap(\Delta_f/\varepsilon)$ be a Laplace random variable. The algorithm that outputs $f(S) + r$ satisfies $\varepsilon$-differential privacy. Moreover, with probability $1-\beta$ it holds that $\lvert r\rvert\le \ln(1/\beta)\Delta_f/\varepsilon$.
\end{lemma}

We also need the following $\mathsf{AboveThreshold}$ algorithm (aka the sparse vector technique)~\citep{dwork2009complexity}.

\begin{algorithm}[!ht]
\label{alg:svt}
\DontPrintSemicolon
    \KwInput{database $S$; privacy parameter $\varepsilon$; threshold $L$; a series of online and adaptively chosen sensitivity-1 queries $q_1,\dots$}
    \KwOutput{a stream of response $a_1,\dots$}
    $\hat{L} \gets L + \lap(2/\varepsilon)$\;
    \For{$i=1,\dots,$}{
        $\hat{q_i}\gets q_i(S) + \lap(4/\varepsilon)$\;
        \uIf{$\hat{q_i} \ge \hat{L}$}{
            $a_i\gets \top$\;
            \textbf{halt}\;
        }\Else{$a_i\gets \bot$\;}
    }
\caption{$\mathsf{AboveThreshold}$}
\end{algorithm}

\begin{lemma}
    Algorithm~\ref{alg:svt} is $\varepsilon$-differently private.
\end{lemma}

If we only want to identify a query with a large value instead of the values of all queries, the report-noisy-max mechanism gives a much better utility guarantee. It can be implemented by adding Laplace noise or directly applying the exponential mechanism~\citep{mcsherry2007mechanism}.

\begin{theorem}[Report-Noisy-Max]
    Let $S$ be a database and $q_1,\dots, q_d$ be $d$ sensitivity-1 queries. There exists an $\varepsilon$-differentially private algorithm that outputs an index $i$ such that \begin{equation*}q_i(S)\ge \max_{j\in[d]}q_j(S) - \frac{2(\ln(d) + \ln(1/\beta))}{\varepsilon}\end{equation*} with probability at least $1-\beta$.
\end{theorem}

The composition property allows us to combine multiple differentially private algorithms into one, even if they are executed adaptively.

\begin{lemma}[Basic Composition~\citep{dwork2006our,dwork2014algorithmic}]
    Let $\A_1: \Z^n\to \R_1$ be an algorithm that satisfies $(\varepsilon_1,\delta_1)$-DP, and for $2\le i\le k$ let $\A_i: \R_1\times \cdots\times \R_{i-1}\times \Z^n\to \R_i$ be an algorithm that satisfies $(\varepsilon_i,\delta_i)$-DP for any given $(r_1,\dots,r_{i-1})\in\R_1\times\cdots\R_{i-1}$. Let $\A$ be an algorithm that
    \begin{enumerate}
        \item Computes $r_1\gets \A_1(S)$;
        \item For each $i=2,\dots,k$, computes $r_i\gets\A_i(r_1,\dots, r_{i-1}, S)$;
        \item Outputs $r_1,\dots, r_k$.
    \end{enumerate}
    Then $\A$ is $(\sum_{i=1}^k\varepsilon_i,\sum_{i=1}^k\delta_i)$-DP.
\end{lemma}

\section{Proofs from Section~\ref{sec:pure}}
\label{sec:proofpure}

\subsection{Proof of Theorem~\ref{thm:pure_oblivious_realizable}}

We use the following DP-OPE algorithm from~\citep{asi2023near}.

\begin{theorem}
\label{thm:dpOPErealizable}
    For any $0<\beta <1/2$, there exists an $\varepsilon$-differentially private algorithm such that with probability $1-\beta$ it has regret
    \begin{equation*}
        O\left(\frac{\log^2 d + \log(T/\beta)\log(d/\beta)}{\varepsilon}\right)
    \end{equation*}
    against oblivious adversaries in the realizable setting.
\end{theorem}

\begin{proof}[Proof of Theorem~\ref{thm:pure_oblivious_realizable}]
    Let $\alpha = \beta = 1/2T$. By Lemma~\ref{lem:boosting}, there exists an $(\alpha,\beta)$-probabilistic representation of $\HH$ with
    \begin{equation*}
        \mathrm{size}(\PP) = O(\log(T)\cdot(\repd(\HH) + \log\log T)).
    \end{equation*}
    Let $S=\{(x_1, y_1),\dots, (x_T, y_T)\}$ be the sequence chosen by the adversary and $\D_S$ be the empirical distribution of $S$, namely, $\Pr_{(x, y)\sim\D_S}[(x, y) = (x_t,y_t)]=1/T$ for all $t\in[T]$. Then we have
    \begin{align*}
        \Pr_{V\sim\PP}[\exists v\in V~s.t.~y_t=v(x_t)~\forall t\in[T]]  &= \Pr_{V\sim\PP}[\exists v\in V~s.t.~\err_{\D_S}\le 1/\alpha=1/2T] \\ &\ge 1-\beta \\ &=1-1/2T.
    \end{align*}
    Conditioning on this event, we then run the algorithm in Theorem~\ref{thm:dpOPErealizable} with $V$ being the set of experts and $\ell_t(v) = 1[v(x_t)\neq y_t]$. With probability $1-1/2T$, the number of mistakes is at most
    \begin{equation*}
        O\left(\frac{\log^2\lvert V\rvert + \log T\log \lvert V\rvert + \log^2 T}{\varepsilon}\right) = O\left(\frac{\log^2T(\repd(\HH) + \log\log T)^2}{\varepsilon}\right).
    \end{equation*}
    By the union bound, the expected number of mistakes is bounded by
    \begin{equation*}
        1/T \cdot T + (1 - 1/T)\cdot O\left(\frac{\log^2T(\repd(\HH) + \log\log T)^2}{\varepsilon}\right).
    \end{equation*}
\end{proof}

\subsection{Proof of Theorem~\ref{thm:pure_oblivious_agnostic}}
The proof follows the same steps as the proof of Theorem~\ref{thm:pure_oblivious_realizable}. The only difference is that we use the following DP-OPE algorithm from~\citep{asi2023private} in the agnostic setting.

\begin{theorem}
\label{thm:dpOPEagnostic}
    There exists an $\varepsilon$-differentially private algorithm that has an expected regret of
    \begin{equation*}
        \mathbb{E}\left[\sum_{t=1}^T\ell_t(i_t) - \min_{i\in[d]}\sum_{t=1}^T\ell_t(i)\right] = O\left(\frac{\sqrt{T}\log d}{\varepsilon}\right)
    \end{equation*}
    against oblivious adversaries in the agnostic setting.
\end{theorem}

\begin{proof}[Proof of Theorem~\ref{thm:pure_oblivious_agnostic}]
    We can use the same argument as in the proof of Theorem~\ref{thm:pure_oblivious_realizable} to sample a hypothesis class $V$ from $\PP$ such that $\ln\lvert V\rvert = O(\log(T)\cdot(\repd(\HH) + \log\log T))$ and
    \begin{equation*}
        \Pr_{V\sim\PP}[\exists v\in V~s.t.~v(x_t) = h^\star(x_t)] \ge 1-1/2T,
    \end{equation*}
    where $h^\star = \mathrm{argmin}_{h\in\HH}\err_{\D_S}(h)$ is a minimizer of the error on $S$. Running the algorithm in Theorem~\ref{thm:dpOPEagnostic} gives an expected regret of at most
    \begin{equation*}
        (1-1/2T)\cdot O\left(\frac{\sqrt{T}\log \lvert V\rvert}{\varepsilon}\right) + T\cdot 1/2T = O\left( \frac{\sqrt{T}\log T(\repd(\HH) + \log\log T)}{\varepsilon}\right).
    \end{equation*}
\end{proof}

\subsection{Proof of Theorem~\ref{thm:pureadaptivelower}}

We start with the following claim, which states that Algorithm~\ref{alg:adaptive_lower} is an adversary that preserves privacy.

\begin{claim}
\label{cla:pure}
    Suppose $\A$ satisfies $\varepsilon$-adaptive differential privacy. Let $\B$ be the algorithm that runs Algorithm~\ref{alg:adaptive_lower} with $\A$. Then $\B$ is $\varepsilon$-differentially private.
\end{claim}
\begin{proof}
    Let $S_1$ and $S_2$ be two input sequences that differ in only one entry. Consider an adversary $\adv$ that runs Algorithm~\ref{alg:adaptive_lower} on $S_1$ and $S_2$ simultanously using the same randomness, and interacts with $\A$ using $S_b$ for some $b\in\{1,2\}$. Let $S_1'$ and $S_2'$ be the resulting sequences. Note that when $S_1[t]=S_2[t]$, with the same randomness we have $S_1'[t] = S_2'[t]$ since the weights depend on $h_1,\dots,h_{t-1}$ only. Thus, $S_1'$ and $S_2'$ also differ in only one entry. By the definition of adaptive differential privacy, it holds that for any event $O$,
    \begin{equation*}
        \Pr[\A\circ\adv(1)\in O]\le e^\varepsilon\Pr[\A\circ\adv(2)\in O].
    \end{equation*}
    The conclusion follows by observing that the output distributions of $\A\circ\adv(b)$ and $\B(S_b)$ are identical.
\end{proof}

\begin{proof}[Proof of Theorem~\ref{thm:pureadaptivelower}]
    Let $\varepsilon \le 0.01$. We first consider $T \in [0.1\ln d/\varepsilon, 0.2\ln d/\varepsilon]$ and prove a lower bound of $\Omega(T) = \Omega(\log d/\varepsilon)$. To this end, we will assume that there exists an online learning algorithm $\A$ with an expected mistake bound of $0.01T$ that is $\varepsilon$-adaptive DP and derive a contradiction.
    
    Let $\B$ be the algorithm that runs Algorithm~\ref{alg:adaptive_lower} with $\A$. By Claim~\ref{cla:pure}, we know that $\B$ is $\varepsilon$-differentially private. Now suppose we are running $\B$ on $S_i = \{(i, 1), \dots, (i, 1)\}$ for some $i\in[d]$. Let $c_t(j) = \sum_{r=1}^t h_r(j)$, $w_t(j)=e^{c_t(j)}$, $\Phi_t = \sum_{j\in[d]\setminus \{i\}}w_t(j)$, and $p_t(j) = w_t(j) / \Phi_t$. The expected number of mistakes made by $\A$ can be expressed as
    \begin{equation}
    \label{equ:pureadaptivemistake}
        \sum_{t=1}^T\mathbb{E}\left[\frac{1}{2}\cdot \I[h_t(i) = 0] + \frac{1}{2}\sum_{j\in[d]\setminus\{i\}}p_{t-1}(j)\cdot h_t(j)\right]\le 0.01T.
    \end{equation}

    Now consider the potential $\Phi_t$. At the beginning, we have $\Phi_0 = d - 1$. We can upper bound $\Phi_t$ by
    \begin{align*}
        \Phi_t &= \sum_{j\in[d]\setminus\{i\}}w_t(j) \\
        &= \sum_{j\in[d]\setminus\{i\}}w_{t-1}(j)e^{h_t(j)} \\
        &= \Phi_{t-1}\sum_{j\in[d]\setminus\{i\}} p_{t-1}(j)e^{h_t(j)}\\
        &\le \Phi_{t-1}\sum_{j\in[d]\setminus\{i\}}p_{t-1}(j)(1+2h_t(j)) \\
        &= \Phi_{t-1}\left(1 + 2\sum_{j\in[d]\setminus\{i\}}p_{t-1}(j)h_t(j)\right) \\
        &\le \Phi_{t-1}\exp\left(2\sum_{j\in[d]\setminus\{i\}}p_{t-1}(j)h_t(j)\right),
    \end{align*}
    where we use $e^x\le 1 + 2x$ for $0\le x\le 1$ in the forth line and $1+x\le e^x$ for $x\in\mathbb{R}$ in the last line. Then it follows by induction that
    \begin{equation*}
        \Phi_T\le \Phi_0\exp\left(2\sum_{t=1}^T\sum_{j\in[d]\setminus\{i\}}p_{t-1}(j)h_t(j)\right).
    \end{equation*}
    Taking the logarithm on both sides, the linearity of expectation gives
    \begin{align*}
        \mathbb{E}[\ln \Phi_T] &\le \mathbb{E}\left[\ln \Phi_0 + 2\sum_{t=1}^T\sum_{j\in[d]\setminus\{i\}}p_{t-1}(j)h_t(j)\right] \\
        &=\ln (d-1) + 2\sum_{t=1}^T\sum_{j\in[d]\setminus\{i\}}\mathbb{E}[p_{t-1}(j)h_t(j)] \\
        &\le \ln (d-1) + 0.04T \\
        &\le 0.14T,
    \end{align*}
    where the third line is due to~\eqref{equ:pureadaptivemistake} and the last inequality uses the facts that $0.1\ln d/\varepsilon\le T$ and $\varepsilon\le 0.01$.

    By Markov's inequality, with probability at least $5/6$ we have $\ln\Phi_T\le 0.84T$. This implies that, for every $j\neq i$, we have
    \begin{equation*}
        c_T(j) =\ln w_T(j)\le\ln\Phi_T \le 0.84T.
    \end{equation*}

    We then bound $c_T(i)$. Note that by~\eqref{equ:pureadaptivemistake} we have
    \begin{equation*}
        \mathbb{E}[T - c_T(i)] =  \mathbb{E}\left[\sum_{t=1}^T\I[h_t(i) = 0] \right] \le 0.02T.
    \end{equation*}
    Applying Markov's inequality again, with probability at least $5/6$ we have $T - c_T(i) \le 0.12T$, or equivalently, $c_T(i)\ge 0.88T$. By the union bound, it holds with probability $2/3$ that $c_T(i)  \ge 0.88T > 0.84T \ge  c_T(j)$ for every $j\neq i$.

    Let $O_i$ be the event that $c_T(i) > c_T(j)$ for every $j\neq i$, then $O_1,\dots, O_d$ are disjoint. Then by group privacy,
    \begin{equation*}
        1\ge \sum_{i=1}^d\Pr[\B(S_1)\in O_i]\ge e^{-T\varepsilon}\sum_{i=1}^d\Pr[\B(S_i)\in O_i]\ge 2/3\cdot de^{-T\varepsilon}.
    \end{equation*}
    Rearranging the inequality yields $T\ge (\ln d - \ln 1.5)/\varepsilon >0.2\ln d/\varepsilon$ when $d \ge 2$, a contradiction.

    Now, let us deal with the remaining range. When $T > 0.2\ln d / \varepsilon$, the algorithm must make $\Omega(\log d /\varepsilon)$ in the first $\lfloor0.2\ln d / \varepsilon\rfloor$ rounds. For $T < 0.1\ln d/\varepsilon$, suppose $\A$ makes no more than $0.005T$ mistakes in expectation. Then we can repeatedly initiate $\A$ after every $T$ round to obtain an algorithm that makes at most $0.005kT$ mistakes in expectation for a total of $kT$ rounds, where
    \begin{equation*}
        k = \left\lfloor \frac{0.2\ln d /\varepsilon}{T}\right\rfloor \ge 0.5\cdot\frac{0.2\ln d /\varepsilon}{T} = \frac{0.1\ln d /\varepsilon}{T}.
    \end{equation*}
    This contradicts our previous conclusion since $kT\in [0.1\ln d /\varepsilon, 0.2\ln d /\varepsilon]$. We thus obtain the $\Omega(T)$ lower bound as desired.
\end{proof}

\section{Proofs from Section~\ref{sec:lower}}
\label{sec:prooflogT}

\subsection{Proof of Theorem~\ref{thm:logT}}

We start by analyzing Algorithm~\ref{alg:Sandt}. As we discussed in Section~\ref{sec:logT}, it constructs a series of data sequences and a list of time-steps that can be used to distinguish the sequences. We formalize this in the following lemma.
\begin{lemma}
\label{lem:algconstruct}
    For any threshold $k$ and online learning algorithm $\A$ such that $\Pr[\A(S_0)_t(u_1) = f_1(u_1)]\ge 1/2$ for all $t\in[T]$, where $S_0 = \{(u_0, f_1(u_0)),\dots,(u_0,f_1(u_0))\}$ and $f_1,f_2,u_0,u_1$ satisfy the property listed in Fact~\ref{fact:non-complementary}, Algorithm~\ref{alg:Sandt} either outputs an $S$ on which $\A$ makes more than $k/3$ mistakes in expectation, or $S_1,\dots, S_m$ and $t_1, \dots, t_{m-1}$ such that
    \begin{enumerate}
        \item $m = \lceil T / k\rceil$.
        \item For each $i\in[m]$, we have $\Pr[\A(S_i)_{t_j}(u_1) = f_1(u_1)] < 1/3$ for all $j\ge i$ and $\Pr[\A(S_i)_{t_j}(u_1)=f_1(u_1)]\ge 1/2$ for all $j \le i - 1$.
    \end{enumerate}
\end{lemma}
\begin{proof}
    If Algorithm~\ref{alg:Sandt} outputs a single $S=S_i$, it inserts at least $k + 1$ $(u_1, f_2(u_1))$'s to $S_i$. By our construction, each of them incurs a mistake with a probability of at least $1/3$. Therefore, the expected number of mistakes made by $\A$ on $S_i$ is at least $(k + 1)/ 3 > k/3$. Now suppose Algorithm~\ref{alg:Sandt} does not output a single data stream, this means for every $i$, Algorithm~\ref{alg:Sandt} replaces at most $k$ elements when constructing $S_i$ from $S_0$.

    For each $i\in[m]$, the algorithm will first find the minimal $t_{i-1}$ such that $\Pr[\A(S_j)_{t_{i-1}}(u_1) = f_1(u_1)]< 1/3$ for all $1\le j \le i - 1$, which is exactly the first half of item 2. Moreover, this also suggests that $t_1\le\cdots\le t_m$. Due to the construction process, the first $t_{i-1}$ entries of $S_i$ will be the same as $S_0$. Therefore, we have $\Pr[\A(S_i)_{t_{j}}(u_1) = f_1(u_1)]\ge 1/2$ for any $1\le j\le i - 1$ since $\A$ is an online algorithm. This proves the second half of item 2.

    Now it remains to show the construction actually runs for $\lceil T / k\rceil$ rounds, i.e., we can find $t_{i-1}$ for any $i\le\lceil T / k\rceil $. Since $S_0$ and any one of $S_1,\dots, S_{i-1}$ differ at most $k$ entries, the pigeonhole principle suggests that there exists some $t\in[(i-1)k + 1]$ such that $S_1[t] = \cdots =S_{i-1}[t] = (u_0, f_1(u_0))$. We next prove that $t > t_j$ for every $1\le j < i-1$. For the sake of contradiction, assume that $j$ is the smallest index such that $t_j > t$ and $j < i - 1$ (it is impossible for $t$ to be equal to any $t_j$ since we always have $S_{j+1}[t_j]= (u_1, f_2(u_1))$ by our construction). Then we have $t_1\le\cdots \le t_{j - 1}\ < t$. This implies $\Pr[\A(S_l)_t(u_1)=f_1(u_1)]< 1/3$ for every $l \le j$, otherwise Algorithm~\ref{alg:Sandt} will put a $(u_1, f_2(u_1))$ at time-step $t$ when constructing $S_l$. By the minimality of $t_j$, we should choose $t_j$ to be $t$ when constructing $S_{j + 1}$, contradicting our assumption that $t_j > t$.
    
    Therefore, we have $t > t_j$ for every $1\le j < i-1$. As we just argued, it follows that $\Pr[\A(S_{l})_t(u_1) = f_1(u_1)]<1/3$ for all $l\in[i-1]$ otherwise Algorithm~\ref{alg:Sandt} will put a $(u_1, f_2(u_1))$ at this location. Thus, we have $t_{i - 1}\le t\le (i - 1)k + 1 \le (\lceil T / k\rceil -1)k + 1 \le T$ by the minimality of $t_{i-1}$. This completes the proof.
\end{proof}

The following lemma suggests that the condition in the Lemma~\ref{lem:algconstruct} can be assumed without loss of generality.

\begin{lemma}
\label{lem:gen}
    Given $f_1, f_2$ and $u_0, u_1$ as in Fact~\ref{fact:non-complementary}. Let  $\A$ be a randomized online learning algorithm and $S_0 = \{(u_0, f_1(u_0)),\dots,(u_0, f_1(u_0))\}$. There exists $b\in\{1,2\}$ such that 
    \begin{equation*}
        \left\lvert\left\{t:\Pr[\A(S_0)_t(u_1)=f_b(u_1)]\ge 1/2\right\}\right\rvert \ge T / 2.
    \end{equation*}
\end{lemma}
\begin{proof}
    Let $n_b = \left\lvert\left\{t:\Pr[\A(S_0)_t(u_1)=f_b(u_1)]\ge 1/2\right\}\right\rvert$ for $b\in\{1,2\}$. We have $n_1 + n_2 \ge T$. Therefore, either $n_1\ge T/2$ or $n_2\ge T/2$.
\end{proof}

\begin{proof}[Proof of Theorem~\ref{thm:logT}]
    Let $\A$ be an $(\varepsilon, \delta)$-differentially private online learning algorithm and $S_0 = \{(u_0, f_1(u_0)),\dots,(u_0, f_1(u_0))\}$. By Lemma~\ref{lem:gen}, we can without loss of generality assume that $\Pr[\A(S_0)_t(u_1)=f_1(u_1)]\ge 1/2$ for all $t\in[T]$ by ignoring all time-steps that do not have such property.

    Let $k = \lfloor c\ln T/(3600\varepsilon)\rfloor$. We will assume that $\A$ makes no more than $k/3$ mistakes in expectation and derive a contradiction. By Lemma~\ref{lem:algconstruct}, Algorithm~\ref{alg:Sandt} will output $S_1,\dots, S_m$ and $t_1,\dots, t_{m-1}$ with $m = \lceil T / k\rceil$. Moreover, $\Pr[\A(S_i)_{t_j}(u_1)=f_1(u_1)] < 1/3$ for every $j\ge i$ and $\Pr[\A(S_i)_{t_j}(u_1)=f_1(u_1)] \ge 1/2$ for every $j\le i-1$.

    Let $\B$ be the algorithm that runs Algorithm~\ref{alg:smoothedbinary} with $\A$. The basic composition property immediately guarantees that $\B$ is $(360\varepsilon, 360\delta)$-differentially private.

    We now examine the probability that $\B$ outputs $i$ on $S_i$. If so, $\B$ must go through a series of time-step $t_{i_1},\dots, t_{i_n}$ in the binary search, where $n \le \lceil\log_2 m\rceil \le \log_2 m + 1 = \log_2 (2m)$. For each $t_{i_j}$, let $X_j$ be a random variable that takes $1$ if 
    \begin{equation*}
        \I\left[\left\lvert\left\{h_{t_{i_j}}^{(w)}(u_1) = f_1(u_1):w\in[360]\right\}\right\rvert < 150\right] = \I\left[i_j \ge i\right].
    \end{equation*}
    and takes $0$ otherwise. Conditioning on $X_j$'s, the probability that $\B$ outputs $i$ can be expressed as
    \begin{equation*}
        \Pr[\B(S_i) = i\left.\right| X_1,\dots,X_n] = \Pi_{j=1}^n\left( 0.75X_j + 0.25(1-X_j)\right) = 0.75^{\sum_{j = 1}^n X_j}0.25^{n - \sum_{j=1}^n X_j}.
    \end{equation*}
    
    By Hoeffding's inequality, for each $X_j$ we have
    \begin{equation*}
        \mathbb{E}[X_j] = \Pr[X_j = 1] \ge 1 - \exp\left(-2\cdot 360\cdot (30/360)^2\right) = 1 - e^{-5} >0.99.
    \end{equation*}
    It then follows by the linearity of expectation and Markov's inequality that 
    \begin{equation*}
        \Pr[X_1+\cdots+X_n \ge 0.97n] \ge 2/3.
    \end{equation*}
    
    Putting it all together gives
    \begin{align*}
        \Pr[\B(S_i) = i] &\ge \Pr[\B(S_i) = i \left.\right| X_1+\cdots + X_n\ge 0.97n]\cdot \Pr[X_1+\cdots + X_n\ge 0.97n] \\
        &\ge 2/3\cdot0.75^{0.97n}\cdot0.25^{0.03n} \\
        &\ge 2/3\cdot\left(\frac{1}{2}\cdot 2^{1/2}\right)^{0.97n} \cdot \left(\frac{1}{2}\cdot 2^{-1}\right)^{0.03n} \\
        &= 2/3\cdot 2^{-n} \cdot 2^{(0.97/2 - 0.03)n} \\
        &= 2/3 \cdot 2^{-0.545n} \\
        &\ge 2/3 \cdot (2m)^{-0.545}.
    \end{align*}

    We now apply the packing argument. Since $\B$ is $(360\varepsilon, 360\delta)$-differentially private, it follows by group privacy that
    \begin{align*}
        1 &= \sum_{i=1}^m\Pr[\B(S_0) = i] \\
        &\ge \sum_{i=1}^m\left(e^{-360k\varepsilon}\Pr[\B(S_i) = i] - \frac{1-e^{-360k\varepsilon}}{e^{360\varepsilon} - 1}\cdot 360\delta\right) \\
        &\ge 2/3\cdot 2^{-0.545}\cdot m^{0.455}\cdot e^{-360k\varepsilon} - m\cdot\frac{1-e^{-360k\varepsilon}}{e^{360\varepsilon} - 1}\cdot 360\delta \\
        &\ge 0.45m^{0.455}\cdot e^{-360k\varepsilon} - m\cdot\frac{1-e^{-360k\varepsilon}}{e^{360\varepsilon} - 1}\cdot 360\delta.
    \end{align*}
    If $0.45m^{0.455}\cdot e^{-360k\varepsilon} \ge 2m\cdot\frac{1-e^{-360k\varepsilon}}{e^{360\varepsilon} - 1}\cdot 360\delta$, then we have
    \begin{equation*}
        0.225m^{0.455}\cdot e^{-360k\varepsilon} \le 0.45m^{0.455}\cdot e^{-360k\varepsilon} - m\cdot\frac{1-e^{360k\varepsilon}}{e^{360\varepsilon} - 1}\cdot 360\delta \le 1.
    \end{equation*}
    Rearranging the above gives 
    \begin{align*}
        k &\ge \frac{\ln0.225 + 0.455\ln m}{360\varepsilon} \\ &\ge \frac{\ln 0.225 + 0.455\ln (T/k)} {360\varepsilon} \\
        &\ge \frac{\ln 0.225 + 0.455 \ln \left(\frac{3600T\varepsilon}{c\ln T}\right)}{360\varepsilon} \\
        &\ge \frac{ \ln 0.225 + 0.455 c\ln T + 0.455 \ln(3600/c)}{360\varepsilon} \\
        &\ge \frac{0.455c\ln T}{360\varepsilon},
    \end{align*}
    where in the forth inequality we use the condition $\varepsilon \ge \ln T/T^{1-c}$. This contradicts our assumption that $k\le c\ln T/3600\varepsilon$.

    If otherwise $0.45m^{0.455}\cdot e^{-360k\varepsilon} < 2m\cdot\frac{1-e^{-360k\varepsilon}}{e^{360\varepsilon} - 1}\cdot 360\delta$, we have
    \begin{align*}
        360k\varepsilon &> \ln\left(\frac{e^{360\varepsilon} - 1}{360\delta}\cdot 0.225m^{-0.545} + 1\right) \\
        &> \ln\left(\frac{e^{360\varepsilon} - 1}{360\delta}\right) + \ln 0.225 -0.545\ln m \\
        &\ge \ln(\varepsilon / \delta) - 0.545\ln T + \ln 0.225 \\
        &\ge 0.455 \ln T + \ln 0.225 \\
        &\ge 0.1\ln T
    \end{align*}
    when $T\ge 100$. Then it follows that $k > 0.1 \ln T /(360\varepsilon) > c\ln T / 3600\varepsilon$, again a contradiction.

    In conclusion, any $(\varepsilon, \delta)$-differentially private algorithm makes $(k+1) / 3 \ge c\ln T/(10800\varepsilon)$ mistakes in expectation when $T\ge 100$. This gives the $\Omega(\log T/\varepsilon)$ lower bound.
\end{proof}

\subsection{Proof of Theorem~\ref{thm:LDlogT}}
\begin{proof}
    When $\ld(\HH) = 1$, the result is directly implied by Theorem~\ref{thm:logT}. From now on, we will assume $\ld(\HH)\ge 2$.
    
    Let $m = \lfloor\ld(\HH) / 2\rfloor$ and $T' = \lfloor T / m\rfloor \ge \lfloor 2T / \ld(\HH)\rfloor \ge T / \ld(\HH) > T^{1 - 1/(1 + c_2)} = T^{c_2 / (1 + c_2)}$. Pick a shattered tree for $\HH$ of depth $\ld(\HH)$. Let $u_0$ be its root and $u_1$ be the left child of $u_0$. The definition of shattered tree indicates that there exists $f_1,f_2\in\HH$ such that $f_1(u_0) = f_2(u_0) = 0$ and $0 = f_1(u_1) \neq f_2(u_1) = 1$. By Theorem~\ref{thm:logT}, for any $(\varepsilon,\delta)$-differentially private online learner with $ \varepsilon \ge \ln T / T^{c_2(1-c_1)/(1+c_2)}> \ln T' / T'^{1 - c_1} $ and $\delta \le \varepsilon / T < \varepsilon / T'$, we can construct a sequence $S_1$ such that it makes in expectation $\Omega(\log T' /\varepsilon) = \Omega(\log T/\varepsilon)$ mistakes on $S_1$. Moreover, $S_1$ only contains $(u_0, 0)$ and $(u_1, b)$ for some $b\in\{0, 1\}$.

    Let $\HH' = \{h\in\HH:h(u_0) = 0~and~h(u_1)=b\}$. If 
    $\ld(\HH) \ge 4$, we then have $\ld(\HH') \ge \ld(\HH) - 2\ge 2$ since the subtree rooted at the child (left if $b = 0$, right if $b=1$) of $u_1$ is shattered by $\HH'$ and has a depth of $\ld(\HH) - 2$. Thus, we can similarly construct a sequence $S_2$ such that the algorithm makes $\Omega(\log T/\varepsilon)$ errors. Importantly, the entries in $S_2$ are completely different from those in $S_1$.

    We repeat the above process until we reach a hypothesis class with Littlestone dimension $\le 1$ and construct a series of sequences $S_1,\dots, S_m$ that have non-overlapping entries. On each of them, the learner makes $\Omega(\log T/\varepsilon)$ mistakes in expectation. Let $S$ be the stream formed by concatenating $S_1,\dots,S_m$ together. The length of $S$ is at most $mT' \le T$ while the expected number of mistakes on $S$ is $\Omega(\ld(\HH)\log T)$.
\end{proof}

\section{Additional Algorithms}
\label{sec:algo}

In this section, we provide several algorithms for various tasks. Note that some of the upper bounds hold even against strong adaptive adversaries, i.e., adversaries that can see the prediction made by the learner in the current round.

\subsection{Learning Point Functions Against Oblivious Adversaries}

We show how to improve the $\log^2 T$ dependence in Theorem~\ref{thm:pure_oblivious_realizable} to $\log T$ for $\point_{\mathbb{N}}$. In the beginning, we keep outputting an all-zero function and use the sparse vector technique to monitor the mistakes. Then we sample hypotheses from the probabilistic representation constructed in~\citep{beimel2019characterizing} and again apply the sparse vector technique to find one with low error on the past data. Following their argument, we can show that the construction guarantees that the hypothesis we found won't make too many mistakes in the future. The result is stated as follows. 
\begin{theorem}
\label{thm:pointN}
    In the realizable setting, there is an $\varepsilon$-differentially private online learning algorithm for $\point_{\mathbb{N}}$ that makes in expectation $O(\log T/\varepsilon)$ mistakes against oblivious adversaries.
\end{theorem}
\begin{proof}
    We run Algorithm~\ref{alg:point}. Note that it is in fact composed of two instances of $\mathsf{AboveThreshold}$ with privacy parameter $\varepsilon/2$, the algorithm is $\varepsilon$-DP.

    Consider the first $\mathsf{AboveThreshold}$, there are $T+1\le 2T$ Laplace random variables. With probability $1-1/3T$, all of them have absolute values no larger than $4\ln(6T^2)/\varepsilon'$. Thus, when it halts, there are at least $\hat{L} - 4\ln(6T^2)/\varepsilon' \ge 20\ln(12T^3)/\varepsilon'$ data points of the form $(x^\star, 1)$, where $f_{x^\star}$ is the target hypothesis. Moreover, it makes at most $20\ln(12T^3)/\varepsilon' + 8\ln(6T^2)/\varepsilon' + 8\ln(6T^2)/\varepsilon' + 1$ mistakes.

    Now we show that the sampling won't last for too long. We consider the first $3T^2$ iterations. Then with probability $1-1/3T$, every random noise has an amplitude of at most $4\ln(12T^3)/\varepsilon'$ in the second $\mathsf{AboveThreshold}$. Note that $\sum_{r=1}^t\I[h(x_r)\neq y_r\text{ and }y_r=1] \ge 20\ln(12T^3)/\varepsilon'$ if $h(x^\star) = 0$ and is $0$ if $h(x^\star) = 1$. Thus, the algorithm will exit the loop if and only if $h(x^\star) = 1$. The probability that this happens in the first $3T^2$ iterations is at least $1-(1-1/T)^{3T^2}\ge 1-e^{-3T} \ge 1- 1/3T$. Moreover, this $h$ makes in expectation less than $(T-1)/T < 1$ mistakes on data points other than $(x^\star, 1)$ in the sequence.

    Putting it all together, the expected number of mistakes is less than
    \begin{equation*}
        1/T\cdot T + (1-1/T)\cdot (20\ln(12T^3)/\varepsilon' + 8\ln(6T^2)/\varepsilon' + 8\ln(6T^2)/\varepsilon' + 1 + 1) = O(\log T/\varepsilon).
    \end{equation*}
\end{proof}

\begin{algorithm}[!ht]
\label{alg:point}
\DontPrintSemicolon
    \KwInput{the number of rounds $T$; privacy parameter $\varepsilon$; sequence $S$}
    \KwOutput{hypotheses $h_1,\dots, h_T$}
    $\varepsilon' \gets \varepsilon / 2$\;
    $\hat{L} \gets 20\ln(12T^3)/\varepsilon' + 8\ln(6T^2)/\varepsilon' + \lap(2/\varepsilon')$\;
    \For{$t=1, \dots, T$}{
        Let $h_t$ be an all-zero function such that $h_t(x) = 0$ for all $x\in\mathbb{N}$\;
        $(x_t, y_t) \gets S[t]$\;
        \If{$\sum_{r=1}^t y_r +\lap(4/\varepsilon')\ge\hat{L}$}{
            $\hat{R} \gets 12\ln(12T^3)/\varepsilon' +\lap(2/\varepsilon')$\;
            \Repeat{$\sum_{r=1}^t\I[h(x_r)\neq y_r\text{ and }y_r=1] +\lap(4/\varepsilon') \le \hat{R}$}{
                Sample an $h$ such that $\Pr[h(x)=1]=1 / T$ for every $x\in\mathbb{N}$ independently\;
            }
            Output $h_{t+1}=\cdots=h_T=h$ for the remaining rounds and halt\;
        }
    }
\caption{Learning point functions over $\mathbb{N}$}
\end{algorithm}

\subsection{Learning Point Functions Against Adaptive Adversaries}

\begin{theorem}
    In the realizable setting, there is an $\varepsilon$-differentially private online learning algorithm for $\point_d$ that makes in expectation $O((\log d + \log T)/\varepsilon)$ mistakes against strong adaptive adversaries.
\end{theorem}
\begin{proof}
    Let $\varepsilon' = \varepsilon / 2$. The algorithm works as follows: it keeps outputting an all-zero function and runs an $\varepsilon'$-differentially private sparse vector technique (Algorithm~\ref{alg:svt}) with $L = 3(\ln d + \ln (2T))/\varepsilon' + 8 \ln(4T^2)/\varepsilon'$ to monitor the number of mistakes. Once the (noisy) number of mistakes exceeds $\hat{L}$ at round $k$, it computes $c_i = \lvert\{t\in [k]: (x_t,y_t)=(i,1)\}\rvert$ and apply the report-noisy-max mechanism with privacy parameter $\varepsilon'$ to find an index $i$ with a large $c_i$. After that, it persistently outputs $f_i$ till the end. The privacy directly follows from the basic composition.
    
    With probability $1- 1/2T$, the amplitude of every noise added in the sparse vector technique is no larger than $4\ln(4T^2)/\varepsilon'$. Thus, at round $k$, the number of mistakes must be in the range $[\hat{L} - 4 \ln(4T^2)/\varepsilon', \hat{L} + 4 \ln(4T^2)/\varepsilon' + 1] \subseteq [L - 8 \ln(4T^2) / \varepsilon', L + 8 \ln(4T^2) / \varepsilon' + 1]$. Hence, we have $c_i \ge L - 8 \ln(4T^2) / \varepsilon' = 3(\ln d + \ln(2T))/\varepsilon'$ for some $i$ and $c_j = 0$ for all $j\neq i$. With probability $1- 1/2T$, the report-noisy-max algorithm will identify $i$ correctly.

    Thus, the expected number of mistakes is bounded by
    \begin{equation*}
        1/T\cdot T + (1 - 1/T)\cdot (L + 8 \ln(4T^2) / \varepsilon' + 1) = O((\log d + \log T) / \varepsilon).
    \end{equation*}
\end{proof}

\subsection{Learning Threshold Functions}

A threshold function over $[d]$ is a function $f_i$ such that $f_i(x) = 0$ for all $x \le i$ and $f_i(x) = 1$ for all $x > i$, where $i\in[d]\cup\{0\}$. The class of threshold functions over $[d]$, denote by $\threshold_d$, is the set $\{f_i:i\in[d]\cup\{0\}\}$.

\begin{theorem}
\label{thm:threshold}
    In the realizable setting, there is an $\varepsilon$-differentially private online learning algorithm for $\threshold_d$ that makes in expectation $O(\log d\log T/\varepsilon)$ mistakes against strong adaptive adversaries.
\end{theorem}
\begin{proof}
    We run Algorithm~\ref{alg:threshold}. Since the counters are refreshed once we switch the current hypothesis, we are actually running $\mathsf{AboveThreshold}$ on disjoint datasets. Moreover, the comparisons of $c_0$ and $c_1$ are also performed on disjoint datasets. Therefore, the overall algorithm is $\varepsilon$-DP.
    
    Since the binary search runs for at most $\lceil\log_2(d+1)\rceil$ iterations, it follows by the privacy of $\mathsf{AboveThreshold}$ and the basic composition that the overall algorithm is $\varepsilon$-DP since the counters are refreshed once we switch the current hypothesis.

    By the property of Laplace distribution and the union bound, with probability $1-1/T$, every random noise that appears in the algorithm has an amplitude no larger than $4\ln(4T^2) / \varepsilon'$. Conditioning on this event, once we change the current hypothesis, we must have $c_0 + c_1 \ge \hat{L} - 4\ln(4T^2) / \varepsilon' \ge 8\ln(4T^2)/\varepsilon'$ and $c_0 + c_1\le \hat{L} + 4\ln(4T^2)/\varepsilon' + 1\le 24\ln(4T^2)/\varepsilon' + 1$. Therefore, we can identify which one is zero. Since the binary search runs at most $O(\log d)$ iterations, we will make at most
    \begin{equation*}
        O(\log d)\cdot (24\ln(4T^2)/\varepsilon' + 1) = O(\log d\log T/\varepsilon)
    \end{equation*}
    mistakes with probability $1-1/T$. This implies an $1/T\cdot T + O(\log d\log T/\varepsilon) = O(\log d\log T/\varepsilon)$ bound in expectation. 
\end{proof}

\begin{algorithm}[!ht]
\label{alg:threshold}
\DontPrintSemicolon
    \KwInput{the number of rounds $T$; privacy parameter $\varepsilon$; sequence $S$}
    \KwOutput{hypotheses $h_1,\dots, h_T$}
    $\varepsilon' = \varepsilon/2$\;
    $\hat{L}\gets 16\ln(4T^2)/\varepsilon'+ \lap(2/\varepsilon')$\;
    $l\gets 0, r\gets d, mid\gets \lfloor (l + r) / 2\rfloor$\;
    $c_0\gets 0, c_1\gets 0$\;
    \For{$t = 1,\dots, T$}{
        $h_t\gets f_{mid}$\;
        $(x_t, y_t)\gets S[t]$\;
        $c_{y_t}\gets c_{y_t} + \I[h(x_t)\neq y_t]$\;
        \If{$l < r\text{ and }c_0 + c_1 + \lap(4/\varepsilon') \ge \hat{L}$}{
            \eIf{$c_0 + \lap(1/\varepsilon') > c_1$}{
                $l\gets mid + 1$\;
            }{$r\gets mid - 1$\;}
            $mid\gets \lfloor(l+r)/2\rfloor$\;
            $c_1\gets 0, c_2\gets 0$\;
            $\hat{L}\gets 16\ln(4T^2)/\varepsilon' + \lap(2/\varepsilon')$\;
        }
    }
\caption{Learning threshold functions over $[d]$}
\end{algorithm}

\subsection{Online Prediction from Experts Against Adaptive Adversaries}

It is easy to come up with an algorithm with a regret of $O(d\log T/\varepsilon)$ even against strong adaptive adversaries. The idea is similar to Algorithm~\ref{alg:threshold}. The only difference is that we try the experts one by one instead of running a binary search.

\begin{theorem}
\label{thm:OPEstrong}
    There is an $\varepsilon$-differentially private algorithm that solves the OPE problem with an expected regret of $O\left(\frac{d \log T}{\varepsilon}\right)$ even against strong adaptive adversaries in the realizable setting.
\end{theorem}
\begin{proof}
    We iterate over the set of experts. For each expert, we keep choosing it and run an $\varepsilon$-differentially private sparse vector technique to monitor the loss incurred by the current expert. Once the sparse vector technique halts, we switch to the next expert and restart the sparse vector technique. Since all instances of the sparse vector technique are run on disjoint data, the entire algorithm is $\varepsilon$-DP.

    With probability $1-1/T$, all the Laplace noises added are bounded by $4\ln(2T^2)/\varepsilon$. By choosing $L = 9\ln(2T^2)/\varepsilon$, we will make at most $O(\log T/\varepsilon)$ mistakes on each expert. Moreover, for an expert that makes no errors, we will not switch to the next one. Therefore, the overall expected regret is
    \begin{equation*}
        1/T\cdot T + (1 - 1/T)\cdot d\cdot O(\log T/\varepsilon) = O(d\log T/\varepsilon).
    \end{equation*}
\end{proof}

Note that there is a multiplicative factor of $d$ on the $\log T$ term. We will then show how to improve this dependence to $\log d$ for (weak) adaptive adversaries. Due to Corollary~\ref{cor:logdlogT}, such a dependence is tight even for oblivious adversaries. It is interesting to find out if this can also be achieved for strong adaptive adversaries, and we leave this as future work.

Moreover, our algorithm has a regret of $O\left(d\log^2 d + \log d \log T\right)$. Since it was shown by~\citet{asi2023private} that an $\Omega(d)$ cost is inevitable, our algorithm is near optimal. 

The idea is to select a uniformly random expert rather than keep choosing a fixed one. Suppose we are at the beginning. Let $c_t(i) = \sum_{r=1}^t \ell_r(i)$ for expert $i$. The benefit of such random selection is that it only incurs a loss of $\sum_{i\in[d]}c_t(i)/ d$, which is much less than the $\sum_{i\in[d]}c_t(i)$ loss if we always choose the same expert since the adversary can let this expert make an error all the time.

We use the sparse vector technique to track the maximum of $c_t$. Once it exceeds $O(\log T + d \log d)$, we then apply the report-noisy-max algorithm to remove every expert $j$ with $a_t(j) = \sum_{r=1}^t \ell_r(j) \ge O(d\log d)$ (we do not write $c_t$ here since we will reset $c_t$ to be $0$ later). After that, we reiterate the sampling and monitoring process using the remaining experts. Observe that if expert $j$ is the $i$-th one being removed, the loss incurred by $j$ is at most $O(\log T + d\log d) / i$. This gives an upper bound of $O(\log T + d\log d)\cdot (1 + 1/2 + \cdots + 1/d) = O(d\log^2 d + \log d\log T)$. The details are depicted in Algorithm~\ref{alg:OPEweak}.

Here is a subtle issue: the report-noisy-max mechanism only succeeds with probability $1 - 1/d$. This does not imply an upper bound in expectation. We cannot simply raise the success probability to $1- 1/T$. Otherwise, the $d\log^2 d$ term will become $d\log d\log T$, which is even higher than the brute force. To fix this, we run an extra sparse vector technique on top of the algorithm to inspect the loss. Once the loss is greater than the upper bound, we then run the algorithm in Theorem~\ref{thm:OPEstrong} for the rest of the rounds. This reduces the expected cost from $T$ to $O(d\log T)$ when it fails, hence successfully bound the expected loss.

\begin{theorem}
    There is an $\varepsilon$-differentially private algorithm that solves the OPE problem with an expected regret of $O\left(\frac{d\log^2 d + \log d \log T}{\varepsilon}\right)$ against adaptive adversaries in the realizable setting.
\end{theorem}
\begin{proof}
    We first show the privacy and utility guarantee for Algorithm~\ref{alg:OPEweak}. For privacy, the sparse vector technique is $\varepsilon/2$-DP. The report-noisy-max outside the while loop will be executed at most $d$ times, and the same for the one inside the while loop. Also, the if clause inside the while loop will also be executed at most $d$ times. All of these cost a privacy budget of $3d\varepsilon_2 = \varepsilon / 2$. Thus, the overall algorithm is $\varepsilon$-DP.

    We next analyze the regret. With probability $1 - 1/T$, all the noises added by $\mathsf{AboveThreshold}$ are bounded by $4\ln(2T^2)/\varepsilon_1$. Therefore, once $c_t(i) \ge \frac{16\ln(2T^2)}{\varepsilon_1} + \frac{3(\ln d + \ln(3d^2))}{\varepsilon_2}$ for some $i\in E$, we must invoke the report-noisy-max mechanism. On the other hand, once we invoke the report-noisy-max mechanism outside the while loop, we must have $c_t(i) \ge \frac{3(\ln d + \ln(3d^2))}{\varepsilon_2}$ for some $i$. With probability $1 - 1/3d$, it always identifies an $i$ with $a_t(i) \ge \frac{3(\ln d + \ln(3d^2))}{\varepsilon_2} - \frac{2(\ln d + \ln(3d^2))}{\varepsilon_2} > 0$ and delete it.

    Now let us move into the while loop. With probability $1 - 1/3d$ the report-noisy-max always returns an $i$ with $a_t(i) > \frac{2\ln(3d^2)}{\varepsilon_2} + \frac{2(\ln d + \ln(3d^2))}{\varepsilon_2} - \frac{2(\ln d + \ln(3d^2))}{\varepsilon_2} = \frac{2\ln(3d^2)}{\varepsilon_2}$ whenever $\max_{j\in[E]}a_t(j) > \frac{2(\ln d + \ln(3d^2))}{\varepsilon_2} + \frac{2\ln(3d^2)}{\varepsilon_2}$. Moreover, the noise added in the if clause is less than $\frac{\ln(3d^2)}{\varepsilon_2}$ for the entire time span with probability $1-1/3d$. Thus, expert $i$ will be removed. Conversely, if it identifies an $i$ such that $a_t(i) = 0$, we will not remove $i$ and exit the loop.

    Let $E_t$ be the set $E$ at the beginning of round $t$ and $E_{T+1}$ be the set $E$ after the algorithm terminates. Suppose $i$ is removed from $E$ at time-step $T_i$ (if it still remains in $E$ at the end, we define $T_i = T + 1$). From the above analysis, we know that with probability $1- 1/d - 1/T$, for every expert $i$ we have $a_{T_i}(i) \le \frac{3(\ln d + \ln(3d^2))}{\varepsilon_2} + \frac{2\ln(3d^2)}{\varepsilon_2} + \frac{16\ln(2T^2)}{\varepsilon_1} + 1$. Conditioning on the event that $a_{T_i}(i)\le M$ for some $M$, we can bound the expected loss by
    \begin{align*}
        \sum_{t=1}^T\mathbb{E}\left[\sum_{i\in E_t}\ell_t(i) / \lvert E_t\rvert \right] &= \sum_{t=1}^T\mathbb{E}\left[\sum_{i\in [d]}\I[i\in E_t]\ell_t(i) / \lvert E_t\rvert \right] \\
        &= \mathbb{E}\left[\sum_{i\in[d]}\sum_{t=1}^T\I[i\in E_t]\ell_t(i) / \lvert E_t\rvert \right]\\
        &= \mathbb{E}\left[\sum_{i\in[d]}\sum_{t=1}^{T_i}\ell_t(i) / \lvert E_t\rvert \right]\\
        &\le \mathbb{E}\left[\sum_{i\in[d]}\sum_{t=1}^{T_i}\ell_t(i) / \lvert E_{T_i}\rvert \right]\\
        &= \mathbb{E}\left[\sum_{i\in[d]}a_{T_i}(i) / \lvert E_{T_i}\rvert \right]\\
        &\le \mathbb{E}\left[ \sum_{i\in[d]}\frac{1}{\lvert E_{T_i}\rvert}\right] \cdot M\\
        &\le \left(\sum_{j=1}^d\frac{1}{j}\right)\cdot M\\
        &\le 2\ln d\cdot M\\
    \end{align*}
    for $d\ge 2$ (the case that $d = 1$ is trivial, so we just ignore it).

    We already have $M = O\left(\frac{d\log d + \log T}{\varepsilon}\right)$ with probability $1-1/d-1/T$. However, this is not enough to show an expected bound. To get an expected bound, we run Algorithm~\ref{alg:OPEweak} with privacy parameter $\varepsilon' = \varepsilon/2$ and an $\varepsilon'$-differentially private $\mathsf{AboveThreshold}$ to monitor $\max_{i\in E_t}a_t(i)$ (note that it won't affect the privacy and utility even if we release $E_t$ publicly) with threshold
    \begin{equation*}
        L =  \left(\frac{2\ln(3d^2)}{\varepsilon'} + \frac{3(\ln d + \ln(3d^2))}{\varepsilon'} + \frac{16\ln(2T^2)}{\varepsilon'} + 1\right) + \frac{4\ln(T^2)}{\varepsilon'}.
    \end{equation*}
    Once $\mathsf{AboveThreshold}$ halts, we then stop Algorithm~\ref{alg:OPEweak} and switch to the algorithm in Theorem~\ref{thm:OPEstrong} with privacy parameter $\varepsilon$.
    
    It is not hard to see that the overall algorithm is $\varepsilon$-DP. For utility, we have that with probability $1-1/T$ $\mathsf{AboveThreshold}$ halts once $\max_{i\in E_t}a_t(i)$ is greater than $\overline{L} = L + \frac{4\ln(T^2)}{\varepsilon'}$ and never halts if the value is always no more than $\underline{L} = L - \frac{4\ln(T^2)}{\varepsilon'}$. Conditioning on this, with probability $1-1/d-1/T$ we have $a_{T_i}\le \underline{L}$ for every $i$ and thus $\mathsf{AboveThreshold}$ never halts. In this case, the expected loss is at most $\underline{L}\cdot 2\ln d = O\left(\frac{d\log^2 d + \log d\log T}{\varepsilon}\right)$. If $\mathsf{AboveThreshold}$ halts, the expected loss incurred before that moment should be at most $(\overline{L} + 1) \cdot 2\ln d= O\left(\frac{d\log^2 d + \log d\log T}{\varepsilon}\right)$. And after that, the regret should be $O\left(\frac{d\log T}{\varepsilon}\right)$ by Theorem~\ref{thm:OPEstrong}. Putting these two cases together gives an expected regret of 
    \begin{align*}
        &\left(1 - \frac{1}{d} - \frac{1}{T}\right)\cdot \underline{L}\cdot 2\ln d+ \left(\frac{1}{d} +\frac{1}{T}\right)\left((\overline{L}+1)\cdot 2\ln d + O\left(\frac{d\log T}{\varepsilon}\right)\right) \\
        ={}& O\left(\frac{d\log^2d + \log d\log T}{\varepsilon}\right)+ \left(\frac{1}{d}+\frac{1}{T}\right)O\left(\frac{d\log T}{\varepsilon}\right)\\
        ={}& O\left(\frac{d\log^2 d + \log d\log T}{\varepsilon}\right).
    \end{align*}
    Thus, the expected regret of the entire algorithm is
    \begin{equation*}
         \left(1-\frac{1}{T}\right)\cdot O\left( \frac{d\log ^2 d + \log d\log T}{\varepsilon}\right)
         + \frac{1}{T}\cdot T = O\left(\frac{d\log^2 d + \log d\log T}{\varepsilon}\right).
    \end{equation*}
\end{proof}

\begin{algorithm}[!ht]
\label{alg:OPEweak}
\DontPrintSemicolon
    \KwInput{the number of rounds $T$; privacy parameter $\varepsilon$; sequence $S$}
    \KwOutput{indices of experts $i_1,\dots,i_T$}
    Set $a_0(i)\gets 0$ and $c_0(i)\gets 0$ for all $i\in[d]$\;
    $E\gets [d]$\;
    $\varepsilon_1 = \varepsilon / 2$, $\varepsilon_2 = \varepsilon / (6d)$\;
    $\hat{L}\gets 8\ln(2T^2) /\varepsilon_1 + 3(\ln d + \ln(3d^2))/\varepsilon_2+ \lap(2/\varepsilon_1)$ \;
    \For{$t=1,\dots, T$}{
        Sample $i_t$ uniformly from $E$\;
        $\ell_t \gets S[t]$\;
        Update $a_t(i)\gets a_{t-1}(i) + \ell_t(i)$ and $c_t(i)\gets c_{t-1}(i) + \ell_t(i)$ for all $i\in[d]$\;
        \If{$\lvert E\rvert > 1\text{ and }\max_{i\in E}c_t(i) + \lap(4/\varepsilon_1) \ge \hat{L}$}{
            Run report-noisy-max with privacy parameter $\varepsilon_2$ on $\{a_t(j)\}_{j\in E}$ and obtain some index $i$\;
            Update $E\gets E \setminus \{i\}$\;
            \While{$i\notin E~\text{and}~\lvert E\rvert > 1$}{
                Run report-noisy-max with privacy parameter $\varepsilon_2$ on $\{a_t(j)\}_{j\in E}$ and obtain some index $i$\;
                \If{$a_t(i) + \lap(1/\varepsilon_2) > \ln(3d^2)/\varepsilon_2$}{
                    Update $E\gets E \setminus \{i\}$\;
                }
            }
            $\hat{L} \gets 8\ln(2T^2) / \varepsilon_1 +3(\ln d + \ln(3d^2))/\varepsilon_2+\lap(2/\varepsilon_1)$\;
            $c_t(i)\gets 0$ for all $i\in[d]$\;
        }
    }
\caption{DP-OPE against adaptive adversaries}
\end{algorithm}

\subsection{Learning Two Complementary Hypotheses}

Let $\HH = \{f_1, f_2\}$ such that $f_1 = 1 - f_2$. We now give an algorithm that achieves a mistake bound of $O(1/\varepsilon)$.

It is not hard to come up with an algorithm that makes $O(1/\varepsilon)$ mistakes with constant probability. Since the incorrect hypothesis makes an error on every input sample, we can output arbitrarily in the first $O(1/\varepsilon)$ rounds. Then we use the $O(1/\varepsilon)$ data points to figure out which hypothesis is the correct one. This can be done using the Laplace mechanism.

However, a constant probability bound does not imply an expected bound. Note that once the Laplace mechanism fails, we may make $\Omega(T)$ errors on the entire sequence. Thus, if we follow the above framework, we have to achieve a success probability of $1-1/T$. This requires $O(\log T/\varepsilon)$ samples, which exceed our target.

To reduce the expected number of errors, we split the entire sequence into buckets of length $O(1/\varepsilon)$. We perform the Laplace mechanism on every bucket. Then for the $i$-th bucket, instead of just using the result on the last bucket to predict, we take the majority over all previous buckets. This makes the fail probability decrease exponentially. The mistake bound then converges to $O(1/\varepsilon)$ as desired.

\begin{algorithm}[!ht]
\label{alg:complementary}
\DontPrintSemicolon
    \KwInput{the number of rounds $T$; hypothesis class $\HH=\{f_1, f_2\}$ such that $f_1$ and $f_2$ are complementary; privacy parameter $\varepsilon$; sequence $S$}
    \KwOutput{hypotheses $h_1,\dots, h_T$}
    $c_1 \gets 0$, $c_2\gets 0$\;
    $s\gets \lceil 2\ln(3)/\varepsilon\rceil$\;
    \For{$t = 1,\dots, T$}{
        $(x_t, y_t)\gets S[t]$ \;
        \uIf{$c_1 > c_2$}{
            $h_t\gets f_1$\;
        }\Else{$h_t\gets f_2$\;}
        \If{$t\bmod s = 0$}{
            \uIf{$\sum_{r=t-s+1}^t\I[f_1(x_r)\neq y_r]+ \lap(1/\varepsilon) < s / 2$}{$c_1\gets c_1 + 1$\;}
            \Else{$c_2\gets c_2 + 1$\;}
        }
    }
\caption{Learning complementary hypotheses}
\end{algorithm}

We describe the procedure in Algorithm~\ref{alg:complementary} and provide a formal statement of our result in the following theorem.

\begin{theorem}
\label{thm:finitecomplementary}
    Let $\HH = \{f_1, f_2\}$ be a hypothesis such that $f_1$ and $f_2$ are complementary. In the realizable setting, there exists an $\varepsilon$-differentially private algorithm that online learns $\HH$ with a mistake bound of $O(1/\varepsilon)$ even against strong adaptive adversaries.
\end{theorem}
\begin{proof}
    It is easy to see that Algorithm~\ref{alg:complementary} is $\varepsilon$-differentially private since we add Laplace noise on disjoint buckets. For each bucket, with probability $2/3$, the noise added to the counter is less than $\ln(3)/\varepsilon < s/2$. This means we can identify the correct hypothesis with probability $2/3$.

    For the $n$-th bucket, we will make $s$ errors only if we wrongly identify the target hypothesis on at least half of the previous buckets. By Hoeffding's inequality, this happens with probability at most
    \begin{equation*}
        \exp\left(-2(n-1)(1/6)^2\right) = \exp(-(n-1) / 18).
    \end{equation*}
    Thus, the expected number of mistakes is bounded by
    \begin{equation*}
        s\sum_{n=1}^{\infty}\exp(-(n-1) / 18) = O(1/\varepsilon).
    \end{equation*}
\end{proof}

\end{document}